\documentclass[twoside]{article}
\usepackage[accepted]{aistats2024}

\usepackage{floatrow}
\usepackage{graphicx}
\graphicspath{{images/}}
\usepackage{caption}
\usepackage{subcaption}
\usepackage{comment}

\usepackage{amsmath}
\usepackage{accents}
\newlength{\dhatheight}

\usepackage{amsfonts}
\usepackage{amsthm}

\newtheorem{proposition}{Proposition}
\newtheorem{condition}{Condition}
\newtheorem{corollary}{Corollary}
\newtheorem{definition}{Definition}
\newtheorem{lemma}{Lemma}
\newtheorem{remark}{Remark}

\usepackage{dirtytalk} 
\usepackage{hyperref} 

\usepackage[round]{natbib}
\usepackage[ruled, vlined, linesnumbered]{algorithm2e}
\usepackage{amsfonts}
\usepackage{amssymb}

\usepackage{wrapfig}

\usepackage{xr-hyper}

\begin{document}
\twocolumn[

\aistatstitle{Robust Approximate Sampling via Stochastic Gradient Barker Dynamics}
\aistatsauthor{ Lorenzo Mauri \And Giacomo Zanella }
\aistatsaddress{Department of Statistical Science\\ Duke University \And  Department of Decision Sciences and BIDSA\\ Bocconi University} 
]

\begin{abstract}
    Stochastic Gradient (SG) Markov Chain Monte Carlo algorithms (MCMC) are popular algorithms for Bayesian sampling in the presence of large datasets. 
    However, they come with little theoretical guarantees and assessing their empirical performances is non-trivial.
    In such context, it is crucial to develop algorithms that are robust to the choice of hyperparameters and to gradients heterogeneity since, in practice, 
    both the choice of step-size and behaviour of target gradients induce hard-to-control  biases in the invariant distribution. 
    In this work we introduce the stochastic gradient Barker dynamics (SGBD) algorithm, 
    extending the recently developed Barker MCMC scheme, a robust alternative to Langevin-based sampling algorithms, to the stochastic gradient framework.
    We characterize the impact of stochastic gradients on the Barker transition mechanism and develop a bias-corrected version 
    that, under suitable assumptions, eliminates the error due to the gradient noise in the proposal.
    We illustrate the performance on a number of high-dimensional examples, showing that SGBD is more robust to hyperparameter tuning and to irregular behavior of the target gradients compared to the popular stochastic gradient Langevin dynamics algorithm.
\end{abstract}

\section{INTRODUCTION}\label{sec:introduction}
Approximating posterior distributions arising from probabilistic models is a challenging computational task, especially in the context of large datasets. 
Standard gradient-based MCMC algorithms \citep{roberts95,duane_et_al_87,neal11} 
require evaluations of the exact target density and its gradient at each iteration, which can be computationally impractical.
Inspired by stochastic optimization \citep{robbins51}, stochastic gradient MCMC (SG-MCMC) algorithms replace the exact target gradient with a computationally cheaper estimate, such as those obtained from a randomly sampled subset the original data.
Since the influential work of \citet{sgld}, which introduced the stochastic gradient Langevin dynamics (SGLD) algorithm, SG-MCMC methods have gained considerable popularity among practitioners seeking to perform approximate Bayesian inferences with large datasets. We refer to  \citet{nemeth2021sgmc} for an overview of SG-MCMC.

Most SG-MCMC methods \citet{sgld,sg_hamiltonian, sgnht, ma_et_al_15} converge to the true posterior distribution if the step-size is appropriately decreased to zero \citep{teh_et_al15}. However, this strategy deteriorates mixing, increasing computational cost.
Practitioners usually 
keep the step-size fixed, which leads to non-negligible and hard-to-diagnose bias in the invariant distribution \citep{incompatibility_hmc, brosse_et_al18}, especially if the step-size is chosen too large or the target distribution is irregular.
Also, adaptive tuning is harder in the stochastic gradient setting relative to standard MCMC \citep{nemeth2021sgmc}, which makes 
robust methods even more
appealing in this context. 
Motivated by these considerations, we develop the stochastic gradient version of the Barker proposal scheme developed in \citet{barker}, which has been shown to enjoy improved robustness to target heterogeneity and hyperparameter tuning relative to classical gradient-based MCMC schemes.

The paper is structured as follows. 
Section \ref{sec:background} sets up notation and provides background on the Barker proposal. 
Section \ref{sec:sgbd} introduces and analyzes the stochastic gradient Barker dynamics (SGBD) algorithm. 
In particular: Section \ref{subsec:bias} analyzes the bias induced by direct use stochastic gradients;
Sections \ref{subsec:c-sgbd}-\ref{sec:noise_tolerance} propose a bias-correction methodology and identify the maximum level of noise that it can tolerate; Section \ref{subsec:e_sgbd} shows how to minimize bias for higher levels of noise. 
Section \ref{sec:simulations} numerically compares SGBD to SGLD. 
Therein, SGBD displays greater robustness to the choice of hyperparameters and to irregular posterior distributions, and in most cases exhibits either comparable or better out-of-sample predictive performance. 
Section \ref{sec:conclusion} discusses future research directions.

\section{BACKGROUND}\label{sec:background}
We consider the task of approximate sampling from a target probability distribution of the form $ \pi(\theta) \propto \exp{(g(\theta))}$ where  $\theta \in \mathbb R^d$ and  $g:\mathbb{R}^d\to\mathbb{R}$. In a classical Bayesian setting with conditional independent data, we have $g(\theta) = \log\left(p(\theta)\right) + \sum_{i=1}^N\log\left(p(y_i\mid x_i, \theta)\right)$, where $p(\theta)$ is the prior distribution of the parameter $\theta$, and $p(y_i\mid x_i, \theta)$ is the likelihood component of the $i$-th data point $y_i$ with covariates $x_i$. Hence, the gradient of $g(\theta)$ 
can be written as the sum of $N$ data points components, $\partial_j g(\theta) =  \sum_{i=1}^N \partial_j g_i(\theta)$,
where $\partial_j g_i(\theta) = \frac{1}{N} \partial_j \log \left( p(\theta)\right) + \partial_j \log \left(p(y_i\mid x_i, \theta)\right)$, and $\partial_j$ stands for the partial derivative with respect to the $j^{th}$ component of $\theta$, i.e.\ $\partial_j g(\cdot) = \frac{\partial}{\partial \theta_j} g(\cdot)$ for $j=1, \dots, d$.

\subsection{The Barker Proposal}\label{sec:barker_proposal}
The Barker proposal \citep{barker,barker_fresh_take} is a first-order approximation of 
a locally-balanced jump process \citep{informed_proposals,power2019accelerated,sun2023discrete}.
The latter are continuous-time $\pi$-invariant jump processes with
generator $Lf(\theta)=\int(f(\theta+w)-f(\theta))J(\theta,\theta+w)dw$ 
defined by the intensity function

\begin{align*}
J(\theta,\theta+w)&=
h\left(\frac{\pi(\theta+w)}{\pi(\theta)}\right)
\prod_{j=1}^d\mu_\sigma(w_j)
&\theta,w\in\mathbb{R}^d\,.
\end{align*}
Above $h:(0,\infty)\to(0,\infty)$ can be any function satisfying $h(t)=th(1/t)$ and $\mu_\sigma(z)=\sigma^{-1}\mu(z/\sigma)$ 
any probability density function (PDF) on $\mathbb{R}$ with scale parameter $\sigma>0$
and a symmetric reference distribution $\mu$.
Taking $h$ as the Barker function, $h(t)=2t(1+t)^{-1}$, and approximating $\pi$ with its first order log-Taylor expansion, $\pi(\theta+w)/\pi(\theta)\approx \exp(\sum_{j=1}^d\partial_j g(\theta)w_j)$ leads to the Barker proposal, whose PDF is 
\begin{align}\label{eq:bark_prop}
Q_B(\theta,\theta+w)&=\prod_{i=1}^d
2p(\partial_jg(\theta),w_j)\mu_\sigma(w_j)
&\theta,w\in\mathbb{R}^d\\
\label{eq:p_def}
p(\delta,z)&=(1 + \exp(-z \delta))^{-1}
&\delta,z\in\mathbb{R}\,.
\end{align}
Since $p(\delta,z)+p(\delta,-z)=1$, $Q_B$ is a product of skew-symmetric distributions \citep{azzalini2013skew}, 
which 
provides full-tractability as well as a straightforward algorithm to sample from $Q_B$ (namely lines 2-7 in Algorithm \ref{alg:barker_proposal}). 
The gradient $\partial_j g(\theta^{(t-1)})$ enters into $Q_B$ as a degree of skewness, as opposed to a linear shift of the mean as for classical Euler-Maruyama based schemes such as the Unadjusted Langevin Algorithm (ULA; see e.g \citealp{roberts1996exponential}). 
It follows that the gradients only influence the direction of the increment under $Q_B$ and not its size, since the distribution of $|w_j|$ is independent of $\partial_j g(\theta)$. 
This decoupling of gradients and increments size leads to an increased robustness to sub-optimal hyperparameter tuning and target heterogeneity, see e.g.\ \cite{barker} for more details and some formal results.
On the other hand, being a first-order approximation to a $\pi$-invariant process, $Q_B$ shares the favourable high-dimensional scaling properties of classical gradient-based MCMC, such as a scaling of order $d^{-1/3}$ as $d$ diverges \citep{roberts95,optimal_design_barker}.

Algorithm \ref{alg:barker_proposal} describes the resulting unadjusted Barker proposal algorithm. 
\begin{algorithm}
\SetAlgoLined
\SetKwInOut{Input}{Input}    
\Input{$\theta^{(0)}\in\mathbb{R}^d, \sigma>0$}
 \For{t =1,\dots, T} {
  \For {j=1, \dots, d \tcp*{Can be parallelized}}{
 Draw $w^{(t)}_j \sim \mu_{\sigma}(\cdot)$\;
 Set $b^{(t)}_j = 1$ with probability $p(\partial_j g(\theta^{(t-1)}), w^{(t)}_j)$, 
 otherwise $b^{(t)}_j = -1$\;
 Update $\theta^{(t)}_j \gets \theta^{(t-1)}_j +  b^{(t)}_j w^{(t)}_j$ \;
 }
 }
 \caption{Unadjusted Barker Proposal}
 \label{alg:barker_proposal}
\end{algorithm}
The key step is the flipping operation at line $5$, where the algorithm flips the sign of the proposed increment with probability $1- p(\partial_j g(\theta), w^{(t)}_j)$.
This operation skews the proposal distribution towards the target $\pi$, since the increment will be $+w^{(t)}_j$ with high probability if $\partial_j g(\theta)w^{(t)}_j$ is high and $-w^{(t)}_j$ otherwise.
Note that Algorithm \ref{alg:barker_proposal} has essentially the same computational cost of standard ULA, where $\theta^{(t)}_j \gets \theta^{(t-1)}_j +  \sigma^2/2\partial_j g(\theta^{(t-1)})+w^{(t)}_j$ and $w^{(t)}_j\sim N(0,\sigma^2)$ for $j=1,\dots,d$.
Both schemes require $O(N)$ operations at each iteration, with the computation of the gradient representing the major computational bottleneck with large datasets, as in other gradient-based schemes.
Previous work have considered the Metropolis-adjusted version of Algorithm \ref{alg:barker_proposal}. Here we consider the unadjusted one, in order to mantain the computational savings induced by \eqref{eq:stochastic_gradient} when moving to the stochastic gradient context. While there are works combining MH schemes with mini-batching \citep{austerity_mcmc, bardenet14}, stochastic-gradient versions of unadjusted schemes are much more common and widely used.

In our experiments below we take $\mu_\sigma$ to be the bimodal distribution $0.5\mathcal N(-\sigma,(0.1\sigma)^2)+0.5\mathcal N(\sigma,(0.1\sigma)^2)$, as recommended in \citep{optimal_design_barker}. 
Note that in algorithmic implementations one can simply take $\mu_\sigma=N(\sigma,(0.1\sigma)^2)$, since the resulting algorithm is equivalent by symmetry (though for the equality \eqref{eq:bark_prop} to be correct one needs the symmetric  version of $\mu_\sigma$).

\section{THE STOCHASTIC GRADIENT BARKER PROPOSAL (SGBD)}\label{sec:sgbd}
In this section we propose and analyze the stochastic-gradient Barker Proposal (SGBD) algorithm. At each iteration, we replace the true gradient with the minibatch estimate
\begin{align} \label{eq:stochastic_gradient}
   \hat{\partial}_j g(\theta)& = \frac{N}{n}\sum_{i \in \mathcal{S}_n} \partial_j g_i(\theta)&j=1,\dots,p,
\end{align}
where $\mathcal{S}_n$ is a subset of $\{1,\dots,N\}$ of size $n \ll N$ sampled uniformly at random, with or without replacement.  
The vanilla version of SGBD (v-SGBD) consists in substituting the gradient in Algorithm \ref{alg:barker_proposal} with the estimate in \eqref{eq:stochastic_gradient}, leading to Algorithm \ref{alg:vSGBD}.
\begin{algorithm}
\SetAlgoLined
\SetKwInOut{Input}{Input}   \Input{$\theta^{(0)}\in\mathbb{R}^d, \sigma>0$}
 \For{t =1,\dots, T} {
 Draw $\mathcal{S}_n \subset \{1, \dots, N\}$ uniformly at random\,
 \For {j=1, \dots, d \tcp*{Can be parallelized}}{
 Draw $w^{(t)}_j \sim N(\sigma, (0.1\sigma)^2 )$ \;
 Compute $\hat{\partial}_j g\left(\theta^{(t-1)}\right)$ as in \eqref{eq:stochastic_gradient}\;
 Set $ b^{(t)}_j = 1$ with probability $p\left(\hat{\partial}_j g(\theta^{(t-1)}), w^{(t)}_j\right)$, otherwise $ b^{(t)}_j = -1$\;
 Update $\theta^{(t)}_j \gets \theta^{(t-1)}_j +  b^{(t)}_j w^{(t)}_j$ \;
 }
 }
 \caption{Vanilla Stochastic Gradient Barker Dynamics (v-SGBD)}\label{alg:vSGBD}
\end{algorithm}
\subsection{Bias of vanilla SGBD}\label{subsec:bias}
Lines $2$, $5$ and $6$ of Algorithm \ref{alg:vSGBD} are equivalent to setting $b^{(t)}_j=1$ with probability $\mathbb E[p(\hat{\partial}_j g(\theta^{(t-1}), w^{(t)}_j)]$, where the expectation is taken with respect to the subsampling mechanism.
Thus, denoting current location and proposed increment as $\theta\in\mathbb{R}^d$ and 
 $z\in\mathbb{R}$ for notational simplicity, Algorithm \ref{alg:vSGBD} effectively replaces the flipping probability, i.e.\ $\hbox{Pr}(b=1)=p(\partial_j g(\theta), z)$ in Algorithm \ref{alg:barker_proposal}, with $\mathbb E[p(\hat{\partial}_j g(\theta), z)]$. 
While \eqref{eq:stochastic_gradient} ensures that $\mathbb E[\hat{\partial}_j g(\theta)] =\partial_j g(\theta)$, 
the non-linearity of $p$ implies that $\mathbb E[p(\hat{\partial}_j g(\theta), z)] \neq p(\partial_j g(\theta), z)$ in general. 
The bias of $p(\hat{\partial}_j g(\theta), z)$
implies that the gradient noise does not balance out and, similarly to other SG-MCMC methods, Algorithm \ref{alg:vSGBD} introduces additional error in the stationary distribution of Algorithm \ref{alg:barker_proposal}. 
In the next section we analyse this bias 
and develop a strategy to reduce it. 
We implicitly assume that the approximation error 
inherent to Algorithm \ref{alg:barker_proposal}, induced by the first-order approximation of the jump process, is of smaller order relative to the one induced by stochastic gradients. The same holds true for most SG-MCMC schemes \citep{teh_et_al15,brosse_et_al18}.

First, we identify the direction of the bias.
We make the following symmetry assumption on the stochastic gradient noise, which we denote as $\eta_\theta := \hat{\partial}_j g(\theta) - \partial_j g(\theta)$, suppressing the dependence on $j$ for brevity. 
\begin{condition}[Symmetry of $\eta_\theta$]\label{cond:symmetry}
\begin{equation} \label{eq:unimodality_symmetry}
    \eta_\theta \overset{d}{=} - \eta_{\theta}, \,
\end{equation}
where $\overset{d}{=}$ denotes equality in distribution.
\end{condition}
\begin{proposition}[Direction of bias]\label{prop:bias_p_hat}
Under Condition \ref{cond:symmetry} we have
    \begin{equation}\label{eq:bias_p_hat}
        \left|p\left(\partial_j g(\theta), z\right) - 0.5\right| \geq \left|\mathbb E\left[ p\left(\hat{\partial}_j g(\theta), z\right)\right]- 0.5\right|.\,
    \end{equation}
\end{proposition}
Proposition \ref{prop:bias_p_hat}, as well as results below, holds for every $\theta \in \mathbb{R}^d$ and $z\in\mathbb{R}$. 
Proofs of all theoretical results are provided in the supplement. 
Proposition \ref{prop:bias_p_hat} shows that, under symmetric noise, the expectation of $p(\hat{\partial}_j g(\theta), z)$ is always shrunk towards $0.5$ relative to its target value $p(\partial_j g(\theta), z)$. 
The practical implication of \eqref{eq:bias_p_hat} is an inflation of the variance of the stationary distribution, as Algorithm \ref{alg:vSGBD} moves less frequently towards a local mode of the distribution relative to Algorithm \ref{alg:barker_proposal}. This is analogous to what happens with other SG-MCMC algorithms: for instance, when the step-size is held fixed, the stochastic gradient noise increases the variance of the invariant distribution of SGLD when no correction is taken into account \citep{vollmer_et_al_16}.

\subsection{Corrected SGBD}\label{subsec:c-sgbd}
In this section we quantify the bias of $p(\hat{\partial}_j g(\theta), z)$ and derive a corrected estimator for $\hbox{Pr}(b=1)=p(\partial_j g(\theta), z)$.
To do so, we assume the stochastic gradient noise to be normally distributed. This is a common requirement in the SG-MCMC theory literature, which is typically justified by assuming the mini-batch size to be sufficiently large and applying a central limit theorem \citep{sg_hamiltonian, sgnht, relativistic_mc_17}.
\begin{condition}[Normality of $\eta_\theta$]\label{cond:normality}
    \begin{equation} \label{eq:normality}
           \eta_\theta \sim \mathcal{N}(0, \tau_\theta^2),\,
    \end{equation}
    for some $\tau_\theta>0$ that can depend on $\theta$ (and on $j$). 
\end{condition}
Under Condition \ref{cond:normality}, we obtain the following tractable approximation to the expectation of $p(\hat{\partial}_j g(\theta), z)$.
\begin{proposition}\label{prop:p_hat_biased}
Under Condition \ref{cond:normality}, we have
    \begin{equation}\label{eq:approx_expectation_p_hat}
    \begin{aligned}
        \left|\mathbb E\left[p\left(\hat{\partial}_j g(\theta), z\right)\right] - p\left(c_{z,\tau_\theta}\partial_j g(\theta), z \right) \right| < 0.019,
    \end{aligned}        
    \end{equation}
    where $c_{z,\tau_\theta} := \frac{1.702}{\sqrt{1.702^2 + z^2 \tau_\theta^2}}$.
 \end{proposition}

\begin{remark}\label{rmk:derivation_corrected}
   
Proposition \ref{prop:p_hat_biased} approximates Algorithm \ref{alg:vSGBD} 
 by Algorithm \ref{alg:barker_proposal} 
 with target gradients shrunk by the multiplicative factor $c_{z,\tau_\theta}<1$.

 This supports the idea that stochastic gradients have the effect of tempering the stationary distribution by a power smaller than $1$, and  suggests the strategy of multiplying them by a factor larger than $1$ to counterbalance the effect induced by the noise. In particular, multiplying $\hat\partial_j g(\theta)$ by $\alpha>1$ inflates its expectation by $\alpha$ and its variance by $\alpha^2$. It turns out that the value $\alpha=1.702(1.702^2-\tau_\theta^2z^2)^{-1/2}$ (despite not depending on $\partial_j g(\theta)$)  makes the expectation of the resulting (corrected) estimator approximately equal to $p$ with the correct partial derivative $\partial_j g(\theta)$. This is formalized in Corollary \ref{prop:unbiased_p_tilde}.  \end{remark}

Following Remark \ref{rmk:derivation_corrected}, we define the 
\textit{corrected estimator} of $p(\partial_j g(\theta), z)$ as $\tilde{p}\left(\hat{\partial}_j g(\theta), z\right)$, where for any $\delta,z\in\mathbb{R}$
\begin{equation}\label{p_tilde}
    \tilde{p}\left(\delta, z\right) :=
    \begin{cases}
        p\left(\frac{1.702}{\sqrt{1.702^2-\tau_\theta^2z^2}}\delta, z\right) &\quad \text{if } |z| < \frac{1.702}{\tau_\theta},\\
        \mathbf{1}\left(\delta z >0\right) &\quad \text{otherwise}
    \end{cases},
\end{equation}
with $\mathbf1(A)$ denoting the indicator function of the event $A$. 
When the value of $\tau_\theta$ is not too large, $\tilde{p}\left(\hat{\partial}_j g(\theta), z\right)$ is an approximately unbiased estimator of $p(\partial_j g(\theta), z)$, as stated in the following corollary.
\begin{corollary}[Approximate unbiasedness of $\tilde{p}$]\label{prop:unbiased_p_tilde}
    Assume Condition \ref{cond:normality} and $\tau_\theta < \max\{1.702/|z|,\bar\tau(\partial_j g(\theta),z)\}$, where
    \begin{equation} \label{eq:breaking_point}
    \bar\tau(\delta,z) =\left| \delta/ \Phi^{-1}\left(\left(1+\exp{(-z\delta)}\right)^{-1}\right)\right|\,
\end{equation}
and $\Phi$ denotes the standard Normal CDF.
Then
\begin{equation}\label{eq:unbiased_p_tilde}
\left|\mathbb{E}\left[\tilde{p}\left(\hat{\partial}_j g(\theta), z\right) \right]- p(\partial_j g(\theta), z) \right| < 0.019.\,
\end{equation}
\end{corollary}
Replacing the naive estimate $p(\hat{\partial}_j g(\theta), z)$ used in v-SGBD with $\tilde{p}(\hat{\partial}_j g(\theta), z)$ leads to 
what we refer to as \emph{corrected SGBD} (c-SGBD).
Note, however, that the corrected estimator requires knowledge of the variance of the gradient noise, $\tau_\theta$. In practical applications, $\tau_\theta$ must be estimated.
To do that we adopt a simple online sample variance estimator, leading to the version of c-SGBD described in Algorithm \ref{alg:cSGBD}. 
\begin{algorithm}
\SetAlgoLined
\SetKwInOut{Input}{input}
    \Input{$\theta^{(0)}\in\mathbb{R}^d, \sigma>0, \beta\in(0,1), \{\hat{\tau}_{j}^{(0)}\}_{j=1,\dots,d}$}   
 \For{t =1,\dots, T}{
 Draw $\mathcal{S}_n \subset \{1, \dots, N\}$ uniformly at random\;
 \For{j=1, \dots, d \tcp*{Can be parallelized}}{
 Compute $\hat{\partial}_j g(\theta^{(t-1)})$ using \eqref{eq:stochastic_gradient}\;

Update $\hat\tau_j^{(t)} \gets (1-\beta)\hat\tau_j^{(t-1)} + \beta \sqrt{\sum_{i\in\mathcal{S}_n}\frac{(\partial_j g_i(\theta^{(t-1)}) - \frac{1}{n}\sum_{i\in\mathcal{S}_n}\partial_j g_i(\theta^{(t-1)}))^2}{n-1}}$\;
 Draw $w_j^{(t)}\sim N(\sigma, (0.1\sigma)^2)$\;
 Set $b_j^{(t)} = 1$ with probability $\tilde{p}\left(\hat{\partial}_j g(\theta^{(t-1)}), w_j^{(t)}\right)$, where $\tau_\theta$ in \eqref{p_tilde} is replaced by $\hat\tau_j^{(t)}$, otherwise $b_j^{(t)} = -1$\;
 Update $\theta_j^{(t)} \gets \theta_j^{(t-1)} + b_j^{(t)}w_j^{(t)}$ \;}
 }
  \caption{Corrected Stochastic Gradient Barker Dynamics (c-SGBD)}
  \label{alg:cSGBD}
\end{algorithm}\!

\begin{remark}
While Condition \ref{cond:normality} is not satisfied in most practical scenarios, it allows to quantify the bias and devise estimators that 
can reduce it when normality holds approximately. See for example Figure \ref{fig:p_hat_simulation}. 
\end{remark}

\begin{remark}
    The results in Proposition \ref{prop:p_hat_biased} and Corollary \ref{prop:unbiased_p_tilde} are based on the  bound 
    $\max_x \left|F(x) - \Phi \left(x/1.702\right)\right| < 0.0095$, see e.g.\ \citet[Section~3.2]{logistic_normal_approx_09}. Here $F(\cdot)$ is the CDF of the logistic distribution and $\Phi(\cdot)$ is the one of the standard Normal distribution.
\end{remark}

\subsection{Noise tolerance of SGBD}\label{sec:noise_tolerance}
There is a maximum amount of noise 
that can be tolerated while still being able to estimate 
$p\left(\partial_j g(\theta), z\right)$
from $\hat{\partial}_j g(\theta)$.
In particular, even under Condition \ref{cond:normality}, if $\tau_\theta$ is too large it is not possible to have 

unbiased estimators of $p\left(\partial_j g(\theta), z\right)$, by which we mean functions $\hat{p}(\hat \delta, z;\tau_{\theta})$ taking values in $[0,1]$ such that
\begin{align}\label{eq:existence_estimator}
        \mathbb{E}\left[\hat{p}(\hat \delta, z;\tau_{\theta})\right]= ~ &p(\delta,z)
        &\hbox{for all }\delta\in\mathbb{R}
    \end{align}
    with expectation taken under $\hat \delta\sim N(\delta,\tau_{\theta}^2)$.

\begin{proposition}[Noise tolerance]
\label{prop:breaking_point}
Assume Condition \ref{cond:normality} and $\tau_\theta > \tau^*$, where $$\tau^*
=
\inf_{\delta\in\mathbb{R}}\bar\tau(\delta,z)
=
4\phi(0)/|z|,\,$$ with $\phi$ denoting the standard Normal PDF.
Then there exist no unbiased estimator of $p\left(\partial_j g(\theta), z\right)$.
\end{proposition}
Note that $[0,1]$-valued unbiased estimators are what is needed in order to implement a stochastic gradient scheme that introduce no further bias in Algorithm \ref{alg:barker_proposal}.
Thus, Proposition \ref{prop:breaking_point} identifies a noise level, $\tau^*$, beyond which it is not possible to implement SGBD without introducing further bias due to stochastic gradients. 
The value of $\tau^*=
4\phi(0)/|z|
\approx
1.596/|z|$ is related to, though slightly smaller than, the upper bound on $\tau_\theta $ we required in Corollary \ref{prop:unbiased_p_tilde} to guarantee approximate unbiasedness of $\tilde{p}(\hat{\partial}_j g(\theta), z)$.
Intuitively, since only approximate unbiasedness is required in Corollary \ref{prop:unbiased_p_tilde}, one can afford slightly larger values of $\tau_\theta$ therein.

We can re-interpret Proposition \ref{prop:breaking_point} in terms of upper bound to the algorithmic step-size: given a noise level $\tau_\theta$, the largest increment one can propose in Algorithm \ref{alg:cSGBD} without introducing further bias due to stochastic gradients is $|z|\leq 1.596/\tau_\theta$ (or $|z|\leq 1.702/\tau_\theta$ if a small controllable bias is allowed as in Corollary \ref{prop:unbiased_p_tilde}).
These results could also be used to devise adaptive versions of SGBD where $\sigma$ is tuned on-the-fly so that $|z|\leq 1.702/\tau_\theta$ occurs with high-probability, in a spirit similar to e.g.\ \citet{sgnht}.
We leave such extensions to future work.

Figure \ref{fig:p_hat_simulation} numerically illustrates these phenomena. 
There we plot $p\left({\partial}_j g(\theta), z\right)$, $\mathbb{E}[p(\hat{\partial}_j g(\theta), z)]$ and $\mathbb{E}[\tilde{p}(\hat{\partial}_j g(\theta), z)]$ as a function of $z$ when $\pi(\theta)$ is the posterior distribution in a high-dimensional logistic regression model with real data and randomly chosen values of $\theta$ and $j$. 
See the supplement for more details on the model and generation of $\theta$ and $j$. 
We observe the value of $\mathbb{E}[p(\hat{\partial}_j g(\theta), z)]$
being close to $p\left({\partial}_j g(\theta), z\right)$ when $|z|$ is small, 
while as $|z|$ increases the shrinkage effect discussed in Proposition \ref{prop:bias_p_hat} becomes evident.
The corrected estimator $\tilde{p}$ successfully reduces the bias up until the tolerance level $|z|\approx 1.702/\tau_\theta$, after which the signal in the stochastic gradient is too weak to successfully estimate the true value of $p$ (Corollary \ref{prop:unbiased_p_tilde} and Proposition \ref{prop:breaking_point}).
\begin{figure}
\includegraphics[width=\textwidth]{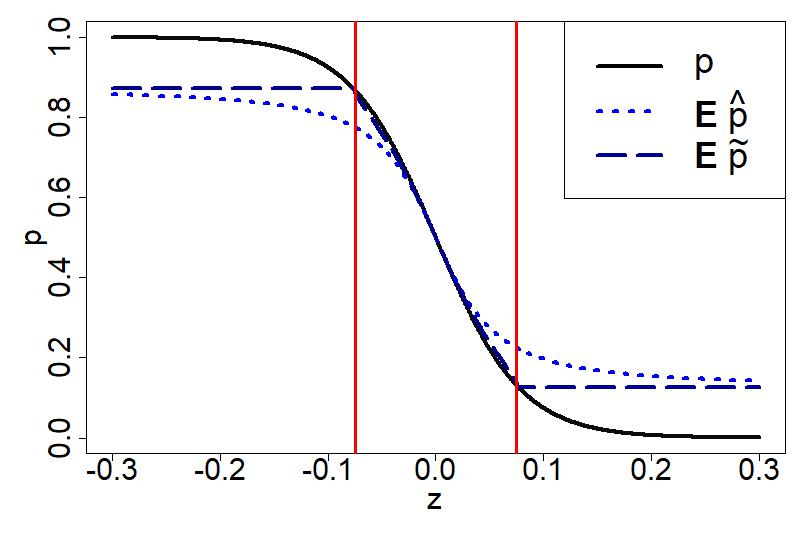}
     \vspace{.3in}
\caption{Shrinkage effect and bias correction. Plot of $p({\partial}_j g(\theta), z)$ (black line; $p$) and Monte Carlo estimates of $\mathbb{E}[p(\hat{\partial}_j g(\theta), z)]$ (dotted blue line; $\textbf{E}\hat{p}$), and $\mathbb{E}[\tilde p(\hat{\partial}_j g(\theta), z)]$ (dashed dark blue line; $\textbf{E}\tilde{p}$) versus the proposed increment $z$; for a logistic regression example with real data (see supplement for more details). Vertical red lines indicate $-1.702/\tau_\theta$ and $1.702/\tau_\theta$.
}
\label{fig:p_hat_simulation}
\end{figure}
We refer to the supplement for a comparison between the tolerance level of SGBD and SGLD.

\subsection{Extreme SGBD}
\label{subsec:e_sgbd}

Sections \ref{subsec:c-sgbd} and \ref{sec:noise_tolerance} show that, 
under Condition \ref{cond:normality}, one can implement a stochastic-gradient version of the unadjusted Barker proposal without introducing significant bias. 
Doing that requires 
 $\tau_\theta<\max \left\{1.702/|z|, \bar \tau (\partial_j g(\theta), z)\right\}$.
This can be achieved either by reducing the stepsize $\sigma$ (which reduces $|z|$) or by increasing the minibatch size $n$ (which reduces $\tau_\theta$).   
However, in many settings, users prefer to run SG-MCMC schemes with larger stepsize and smaller mini-batch size to speed-up convergence and reduce computational cost, even if this introduces non-negligible bias.
In this section we thus focus on the case of larger values of $\tau_\theta$ 

and identify the optimal estimator of $p\left({\partial}_j g(\theta), z\right)$ in such settings, which turns out to be
$$\bar{p}(\delta, z) := \mathbf{1}\left(\delta z >0\right)\,,$$
i.e.\ $\bar{p}(\delta, z)$ equals $1$ when $\delta$ and $z$ have the same sign and $0$ otherwise.
We refer to $\bar p$ as \emph{extreme} estimator. 
Note that $\bar p$ coincides with the \textit{corrected}-estimator, $\tilde p$, whenever $\tau_\theta \geq 1.702/|z|$.
We will prove optimality of $\bar p$ within the following class of estimators.

\begin{definition}[Symmetric estimator]
\label{def:symmetric_estimator}
 An estimator $\hat{p}(\hat \delta, z)=\hat{p}(\hat \delta, z;\tau_\theta)$ of $p\left({\partial}_j g(\theta), z\right)$ is said to be symmetric if 
\begin{align}\label{eq:symmetric_estimator}
  \hat{p}(\hat \delta, z) + \hat{p}(-\hat \delta, z) &= 1
  &
  \hbox{for all }(\hat \delta,z)\in\mathbb{R}^2
   \end{align}
\end{definition}
Condition \eqref{eq:symmetric_estimator} requires $\hat{p}$ to have a symmetric behaviour about 0.5 when the sign of the stochastic gradient is flipped, i.e.\ $\hat{p}(\hat \delta, z) - 0.5 = 0.5 - \hat{p}(-\hat \delta, z)$. 

The latter is a natural requirement for an estimator of $p({\partial}_j g(\theta), z)$ to be of practical interest,
otherwise the resulting algorithm would be unjustifiably biased towards the left or right side of the current value of $\theta_j$.

We make the following assumption on stochastic gradients, which is strictly  weaker than Condition \ref{cond:normality}. 
\begin{condition}[Unimodality of noise distribution]\label{cond:unimodality}
The random variable $\eta_{\theta}$ admits a density function $f_{\theta}(x)$ with respect to the Lebesgue measure and $f_{\theta}(x)$
is non-decreasing for $x\leq 0$ and non-increasing otherwise. 
\end{condition}

Relative to Condition \ref{cond:normality}, Conditions \ref{cond:symmetry} and \ref{cond:unimodality} accomodate for more general scenarios, such as heavier-tailed distributions of $\eta_{\theta}$.

Under these conditions, any symmetric estimator induces more shrinkage towards $0.5$ than $\bar{p}$ as we now show. 
\begin{proposition}\label{prop:optimality_p_bar}
Under Conditions \ref{cond:symmetry} and \ref{cond:unimodality}, we have
    \begin{equation}
        \left|\mathbb E\left[\bar{p}\left(\hat{\partial}_j g(\theta), z\right)\right] - 0.5\right| \geq \left|\mathbb E\left[\hat{p}\left(\hat{\partial}_j g(\theta), z\right)\right] - 0.5\right|.\,
    \end{equation}
for any symmetric estimator $\hat{p}$, with strict inequality when $\hat p\neq \bar p$ and $\partial_j g(\theta)z \neq 0$. 
\end{proposition}

Since $\hat{p}$ is always biased towards $0.5$ for large $\tau_\theta$, Proposition \ref{prop:optimality_p_bar} implies that in such cases $\bar{p}$ achieves minimal bias. 

We formally state this under Condition \ref{cond:normality}, quantifying how large $\tau_\theta$ needs to be.
\begin{corollary}\label{cor:optimality_p_bar}
Assume Condition \ref{cond:normality}, $\partial_j g(\theta)z \neq 0$, and $\tau_{\theta}>\bar\tau(\partial_j g(\theta),z)$, with $\bar\tau$ defined in \eqref{eq:breaking_point}.Then
\begin{align*}
    &\left|p(\partial_j g(\theta), z) - \mathbb E\left[\bar{p}\left(\hat{\partial}_j g(\theta), z\right)\right]\right| \\ &<
    \left|p(\partial_j g(\theta), z) -  \mathbb E\left[\hat{p}\left(\hat{\partial}_j g(\theta), z\right)\right]\right|
\end{align*}
for any symmetric estimator $\hat{p}\neq \bar p$.
\end{corollary}
Corollary \ref{cor:optimality_p_bar} supports the use of $\bar{p}$ for large values of $\tau_\theta$.

Also, Corollaries \ref{prop:unbiased_p_tilde}  and \ref{cor:optimality_p_bar} combined imply that the \textit{corrected} estimator $\tilde{p}(\hat{\partial}_j g(\theta), z)$ is optimal, in the sense of being approximately unbiased when $\tau_\theta \leq \bar\tau(\partial_j g(\theta),z)$ and achieving minimal bias when $\tau_\theta > \bar\tau(\partial_j g(\theta),z)$.
In our simulations, especially when we are interested in predictive accuracy, we also study the performance of the algorithm that always employs the extreme estimator, irrespective of the value of $z$ and $\tau_\theta$.
We refer to such algorithm as \textit{extreme} SGBD (e-SGBD), and provide  pseudo-code for it in the supplement. 
Effectively, e-SGBD sets $b=1$ when $z$ and $\hat \partial_jg(\theta^{(t-1)})$ have the same sign and $b=-1$ otherwise. Hence, it always moves each coordinate in the direction of its component of the stochastic gradient, in a way that is similar to stochastic optimization methods such as AdaGrad \citep{adagrad}. 
\section{EXPERIMENTS}\label{sec:simulations}
In this section we study the performances of the three SGBD versions (vanilla, corrected and extreme), applying them to sampling tasks arising from various models and comparing them to SGLD. 
To help comparability, we consider three versions of SGLD: a vanilla one (v-SGLD) corresponding to the stochastic gradient version of ULA; a corrected one (c-SGLD; see Algorithm 4 in the supplement), where the standard deviation of the artificial noise is adjusted to correct for the stochastic gradient noise; and an extreme one (e-SGLD), where the maximum correction is applied by adding no artificial noise (which simply corresponds to the stochastic gradient descent algorithm, \citealp{robbins51}).
We also tested the modified variant of SGLD proposed in \citep{vollmer_et_al_16}, obtaining comparable results to c-SGLD.
Full details on the SGLD variants, as well as additional simulations for this section can be found in the supplement.
Overall, the results support the fact that the increased robustness of the Barker scheme to target heterogeneity and hyperparameters tuning \citep{barker} is relevant also in the stochastic gradient context.

\subsection{Skewed target distributions}\label{subsec:toy_example}
First, we study how skewness in the target distribution affects algorithmic performances.
Skewness naturally generates heterogeneity in the magnitude of the gradient $\partial_j g(\theta)$ in different regions of the state space, thus being an interesting test case for SG-MCMC schemes.
In particular, we consider the family of skew-normal target  distributions, i.e.\ $\pi_\alpha(\theta) = 2 \phi(\theta) \Phi(\alpha \theta)$ where $\theta\in\mathbb{R}$, $\alpha>0$ and $\phi$ and $\Phi$ are the standard normal PDF and CDF, 
for different values of the skewness parameter $\alpha$.
$\eta_\theta\sim N(0,var(\pi_\alpha))$ and consider two values for the step-size $\sigma$, namely  $\sigma_1=0.1\times sd(\pi_\alpha)$ and $\sigma_2=0.5\times sd(\pi_\alpha)$, where $sd(\pi_\alpha)$ and $var(\pi_\alpha)$ denote the standard deviation and variance of $\pi_\alpha$.  
Figure \ref{plot:toy_1} displays the results for v-SGBD and v-SGLD. 
The results obtained with  corrected variants, as well as different noise levels for $\eta_\theta$, led to similar conclusions and are reported in the supplement.
\begin{figure}[h!]
     \centering
     \begin{subfigure}[b]{0.6\textwidth}
         \centering
         \includegraphics[width=\textwidth]{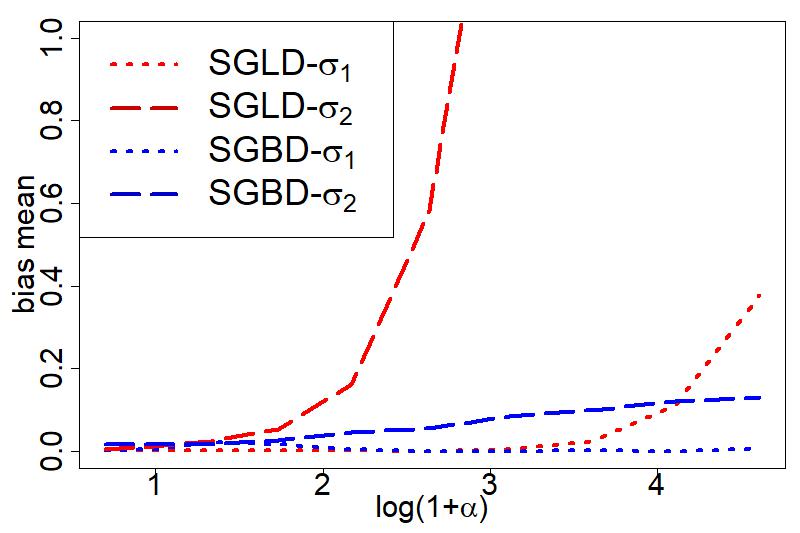}
         \caption{}
         \label{fig:toy_1_a}
     \end{subfigure}
     \hfill
     \begin{subfigure}[b]{0.36\textwidth}
         \centering
         \includegraphics[width=\textwidth, height=1.6cm]{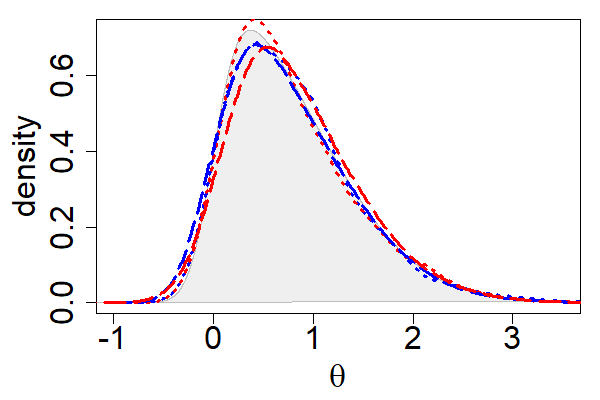}
         \includegraphics[width=\textwidth, height=1.6cm]{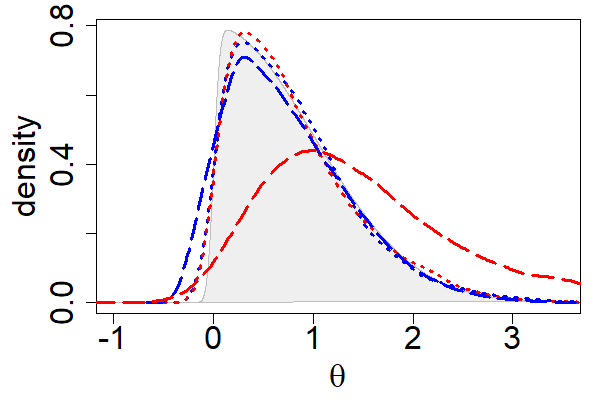}
         \caption{}
          \label{fig:toy_1_b}
          
     \end{subfigure}
          \vspace{.3in}
        \caption[Toy Example: Skew-Normal target with isotropic Gaussian Noise]{Toy Example: Skew-Normal target with isotropic Gaussian Noise. Shape parameter (in log-scale) vs relative bias of mean (left) and invariant distribution of the samplers for different levels of $\alpha$: 5 (top) and 20 (bottom) (right). Red refers to vanilla Langevin-based schemes and blue to vanilla Barker-based schemes. Dotted (dashed resp.) lines are produced with $\sigma_1=0.1\times sd(\pi_\alpha)$ ($\sigma_2=0.5\times sd(\pi_\alpha)$ resp.) of the target standard deviation. The grey shaded area in the right plots represents the true target distribution density.}
        \label{plot:toy_1}
\end{figure}
Figure \ref{fig:toy_1_a} reports the bias on the posterior mean estimates as a function of the skewness parameter $\alpha$, while Figure \ref{fig:toy_1_b} displays stationary distribution for two values of $\alpha$. 
The results suggest that SGBD is more robust to skewness relative to
SGLD, which can suffer from high bias in the invariant distribution as $\alpha$ increases.
Also, the two algorithms exhibit a different robustness to hyperpameter tuning, with SGBD displaying more stable performances when the step-size is increased.

\subsection{Scale Heterogeneity}\label{subsec:log_reg_sepsis}
Next, we study the performance of SGBD on a binary regression task with scale heterogeneity across parameters.

We consider a Bayesian logistic regression model of the form
\begin{equation}\label{eq:log_reg_model}
     y_i\mid x_i, \theta \sim \textit{Bern}\big(\big(1+e^{(-\theta^{\top}x_i)} \big)^{-1}\big) \quad i = 1, \dots, N\,,
\end{equation}
where $x_i \in \mathbb{R}^d$ and $y_i \in \{0, 1\}$. We use SG-MCMC to target the posterior distribution $\pi(\theta)=p(\theta\mid x,y)$ resulting from \eqref{eq:log_reg_model} and a standard normal prior on $\theta$, i.e.\ $\theta\sim \mathcal{N}_d(0,\mathbf{I}_d)$. 

We consider the \href{https://archive.ics.uci.edu/ml/datasets/Sepsis+survival+minimal+clinical+records}{Sepsis dataset} from the UCI repository, which contains $N=110204$ instances and $d=4$ covariates. 
Note that this example is low-dimensional, while in the supplement a high-dimensional binary regression data-set is considered.

The Sepsis dataset leads to an ill-conditioned posterior distribution, where
the posterior standard deviation of the first coordinate is much smaller than the others, as shown in Figure \ref{traceplots:log_reg_sepsis}.
While similar issues related to ill-conditioning could be in principle solved by preconditioning, finding good posterior preconditioners before running MCMC is not always easy and in practice SG-MCMC schemes are often used without adaptive preconditioning, also  
due to the difficulty in tuning them \citep{nemeth2021sgmc}. 
In this sense, robustness to scale heterogneity is a desirable feature for SG-MCMC schemes.

We compare v-SGBD and v-SGLD, with a mini-batch size of $n=\lfloor 0.01 N\rfloor$ and $T=2\times 10^5$ iterations.

Figure \ref{traceplots:log_reg_sepsis} reports the resulting traceplots in the stationary phase. Black lines correspond to the true posterior mean plus and minus two standard deviations, computed with full-batch MCMC using Stan \citep{stan}.

Due to ill-conditioning, SGLD struggles to explore the posterior: tuning the step-size to match the first coordinate (Figure \ref{traceplots:log_reg_sepsis} top row) leads to very poor mixing in the other coordinates; while doubling the step-size with respect to that (Figure \ref{traceplots:log_reg_sepsis} bottom row) dramatically inflates the variance of the first component and leads to biased inferences.  
Instead SGBD, while still being affected by ill-conditioning, is remarkably more robust to scale heterogeneity (e.g.\ it simultaneously achieves good mixing and moderate bias in the stationary distribution) as well as to the choice of hyperparameters (e.g.\ doubling the step-size from top to bottom only leads to a moderate inflation of the marginal stationary distribution in the first component). 
The specific values for the hyperparameters used in Figure \ref{traceplots:log_reg_sepsis} are $\sigma=0.00075, 0.0015$ for SGBD and $\sigma = 0.0002, 0.0004$ for SGLD.
See also the supplement for plots of marginal density estimates.  
\begin{figure}
     \centering
     \begin{subfigure}[b]{0.48 \textwidth}
         \centering
         \includegraphics[width=\textwidth]{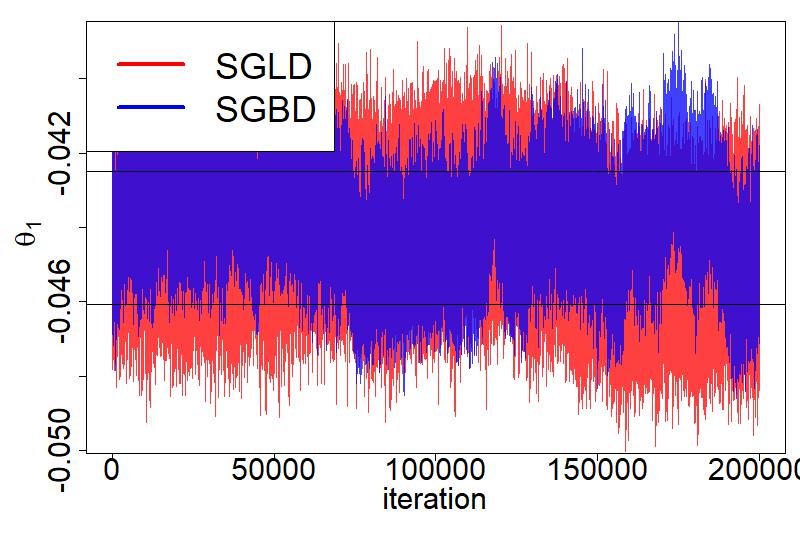}
     \end{subfigure}
     \hfill
     \begin{subfigure}[b]{0.48\textwidth}
         \centering
         \includegraphics[width=\textwidth]{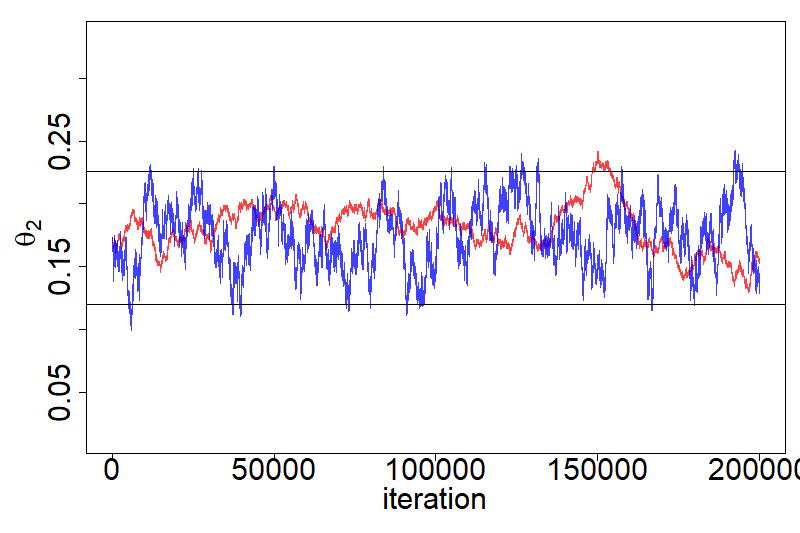}
     \end{subfigure}
     \vfill
     \begin{subfigure}[b]{0.48\textwidth}
         \centering
         \includegraphics[width=\textwidth]{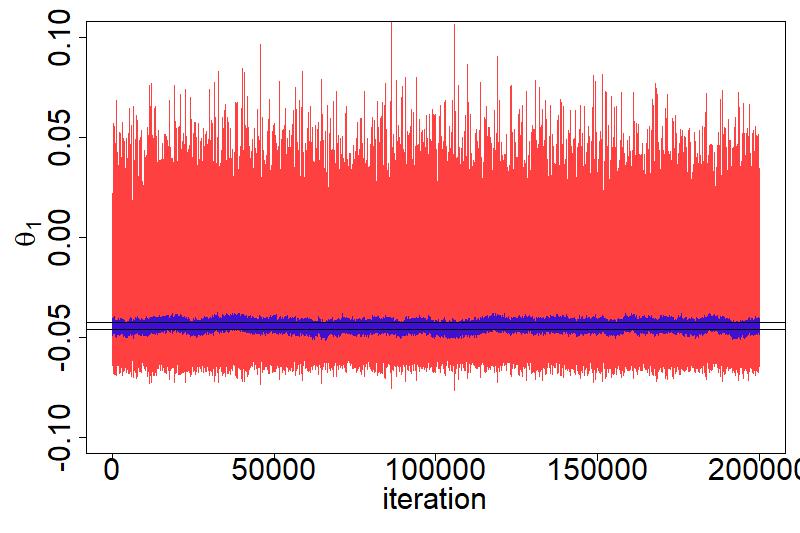}
     \end{subfigure}
     \hfill
     \begin{subfigure}[b]{0.48\textwidth}
         \centering
         \includegraphics[width=\textwidth]{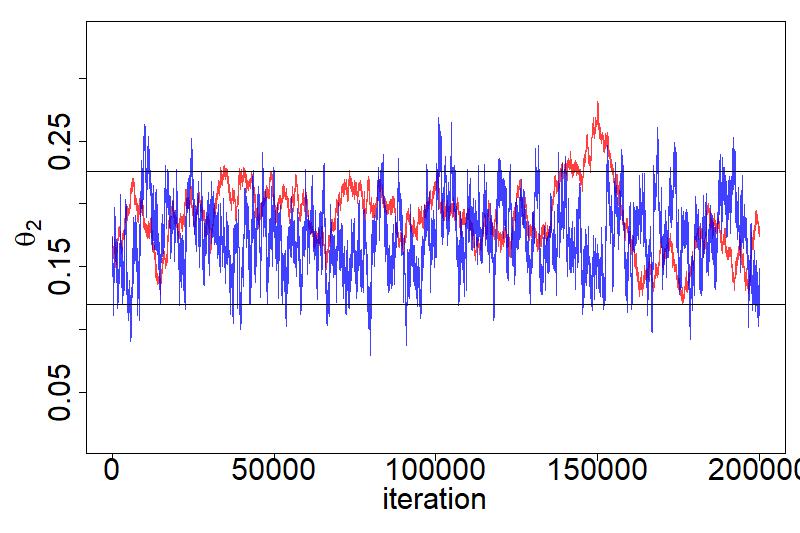}
     \end{subfigure}
     \vspace{.3in}
        \caption[Logistic Regression with Scale Heterogeneity]{
        Sepsis dataset example. Traceplots of two coordinates with a different scale ($\theta_1$ on the left has small scale) with two step-size configuration: small $\sigma$ (top), larger $\sigma$ (bottom). Red refers to v-SGLD and blue to v-SGBD. Black horizontal lines indicate the interval centered at the posterior mean with two standard deviations width.}
        \label{traceplots:log_reg_sepsis}
\end{figure}

In the supplementwe fit model \eqref{eq:log_reg_model} to the Arrhythmia data-set which contains 100 features and 362 training data points. In this example, SGLD is more accurate for small step-sizes and low mixing, while SGBD appears to be more robust to hyperparameter tuning. In particular,
it enjoys a more favourable mixing-accuracy trade off when the step size is chosen increasingly large. In addition, SGBD obtains a better prediction accuracy on the hold-out set for all hyperparameter configurations considered.

\subsection{Bayesian Matrix Factorization}\label{sec.pmf}

We consider a Bayesian matrix factorization model \citep{bpmf} for recommendation. We have $U$ users rating $M$ items. The (potentially sparse) matrix of ratings $\mathbf{R}=(R_{ij})_{ij}$, where $R_{ij}$ corresponds to the rating of user $i$ to item $j$, is modeled as

\begin{equation*}
\begin{aligned}
     R_{ij}&\mid \mathbf{U}, \mathbf{V}, \alpha, I_{ij}=1 \sim \mathcal{N}(\mathbf{U}_i^\top \mathbf{V}_j, \alpha^{-1}) \\
    \mathbf{U}_i&\mid\mu_{\mathbf{U}}, \Lambda_{\mathbf{U}} \sim  \mathcal{N}_p(\mu_{\mathbf{U}}, \Lambda_{\mathbf{U}}^{-1})
    \qquad i = 1, \dots, U\\
\mathbf{V}_j&\mid\mu_{\mathbf{V}}, \Lambda_{\mathbf{V}} \sim \mathcal{N}_p(\mu_{\mathbf{V}}, \Lambda_{\mathbf{V}}^{-1}) 
\qquad j = 1, \dots, M\end{aligned}
    \label{bpmf_model}
\end{equation*}
where $\mathbf{U} \in \mathbb{R}^{U \times p}$, $\mathbf{V} \in \mathbb{R}^{M \times p}$ and $I_{ij}=1$ if user $i$ has rated item $j$.
We adopt the hyperprior
\begin{equation*}
\mu_{\mathbf{Z}}\mid \mu_0,  \Lambda_{\mathbf{Z}} \sim \mathcal{N}_p(\mu_0, \Lambda_{\mathbf{Z}}^{-1}) 
\end{equation*}
where $\Lambda_{\mathbf{Z}} = \text{diag}\left(\lambda_{\mathbf{Z},1}, \dots, \lambda_{\mathbf{Z},p}\right)$, $\lambda_{\mathbf{Z},j}{\sim}\Gamma(a_0, b_0)$ for $j=1, \dots, p$ and $ \mathbf{Z} =(\mathbf U, \mathbf V)$. 

We consider the \href{https://grouplens.org/datasets/movielens/100k}{MovieLens dataset} containing $10^5$ ratings (taking values in $\{1,2,3,4,5\}$) of $1000$ users on $1700$ movies. Taking $p=20$ and applying a $80\%-20\%$ train-test split, we obtain a $54080$-dimensional target posterior distribution $p\left(\mathbf U, \mathbf V, \mu_{\mathbf{U}}, \mu_{\mathbf{V}}, \Lambda_{\mathbf{U}}, \Lambda_{\mathbf{V}} \mid \mathbf R\right)$. We set $\alpha=3$, $\mu_0=0$, $a_0=1$ and $b_0=5$ and use a mini batch size of $n=N/100=800$.

\begin{figure}[h!]
     \centering
     \begin{subfigure}[b]{0.48\textwidth}
         \centering
         \includegraphics[width=\textwidth]{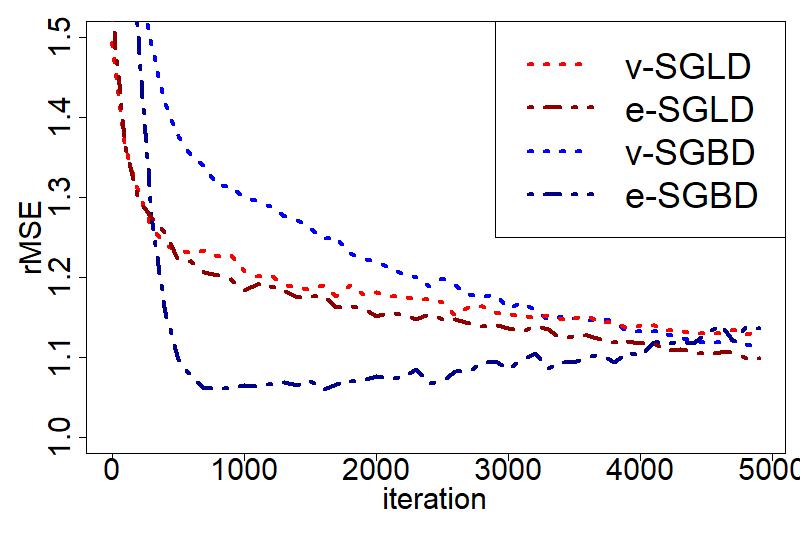}
         \caption{Sample rMSE}
         \label{bpmf_sample_rmse_plot}
     \end{subfigure}
     \hfill
     \begin{subfigure}[b]{0.48\textwidth}
         \centering
         \includegraphics[width=\textwidth]{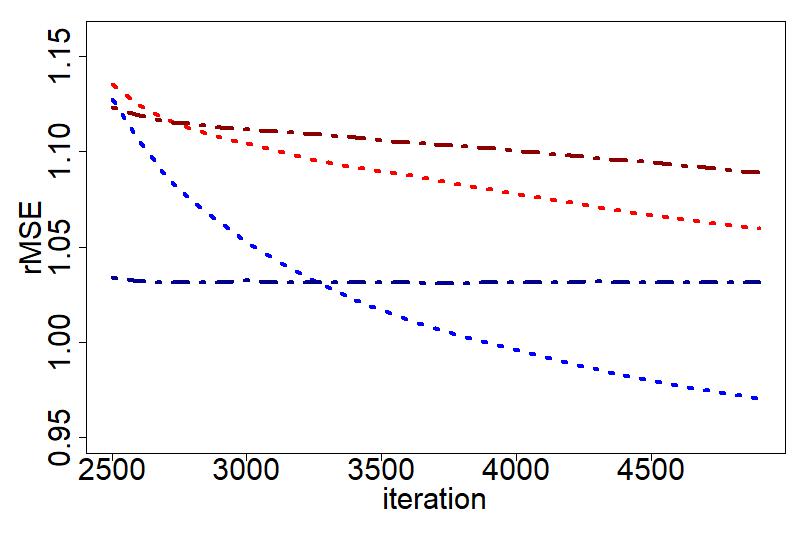}
         \caption{MCMC rMSE}
         \label{bpmf_mcmc_rmse_plot}
     \end{subfigure}
     \vspace{.3in}
        \caption[Bayesian Probabilistic Matrix Factorization: Predictive Accuracy]{Bayesian Probabilistic Matrix Factorization: Predictive Accuracy on the MovieLens dataset. Sample (left) and MCMC (right) estimates rMSE. Red refers to SGLD and blue to the SGBD. Lighter and dotted lines refer to vanilla implementations of the algorithm, darker and dashed-dotted lines to their extreme variants.}
        \label{fig:bpmf}
\end{figure}

Figure \ref{fig:bpmf} compares v-SGBD, e-SGBD, v-SGLD and e-SGLD in terms of predictive performances, measured in root mean squared errors (rMSE; see supplement for explicit definition).
We consider both predictions obtained using single MCMC samples at a given iteration, as well as
ones obtained using MCMC ergodic averages.
In general, SGBD outperforms SGLD in all variants. e-SGBD converges faster and single samples have a better predictive accuracy. However v-SGBD ergodic average estimates have the best overall predictive accuracy. For all schemes, the step-size maximising predictive performances was selected over a grid.

\subsection{Independent Component Analysis}\label{subsec:ica}
We consider an independent component analysis \citep{ica_95,sgld} model, where the likelihood of each datapoint $x_i=(x_{ij})_{j=1,\dots,p} \in \mathbb{R}^p$ for $i=1,\dots,N$ is
\begin{equation*}   \label{eq:ica_model}
    p(x_i\mid W) \propto |\det W| \prod_{k=1}^p 
    \cosh(0.5\sum_{j=1}^pw_{kj} x_{ij})^{-2}
\end{equation*}
and each entry of $W=(w_{ij})_{i,j=1,\dots,p}$ is assigned a standard normal prior, i.e.\ $w_{ij}\stackrel{iid}
\sim N(0,1)$.

We apply the model to the \href{http://research.ics.aalto.fi/ica/eegmeg/MEG\_data.html}{MEG data} collected by the Brain Research Unit, from the Helsinki University of Technology, which contains $N=17730$ data points of dimensionality 100. To perform our experiments, we extract the first $10$ channels (i.e.\ $p=10$) and use SG-MCMC to sample from the resulting $100$-dimensional posterior distribution $p(W|(x_i)_{i=1,\dots,n})$.
We perform a $80\%-20\%$ train-test split and compare log-likelihood on the unseen test set. 
We run the samplers for $T=4\times 10^4$ iterations choosing a batch-size of $n=100$.
\begin{figure}
     \centering
     \begin{subfigure}[b]{0.48\textwidth}
         \centering
         \includegraphics[width=\textwidth]{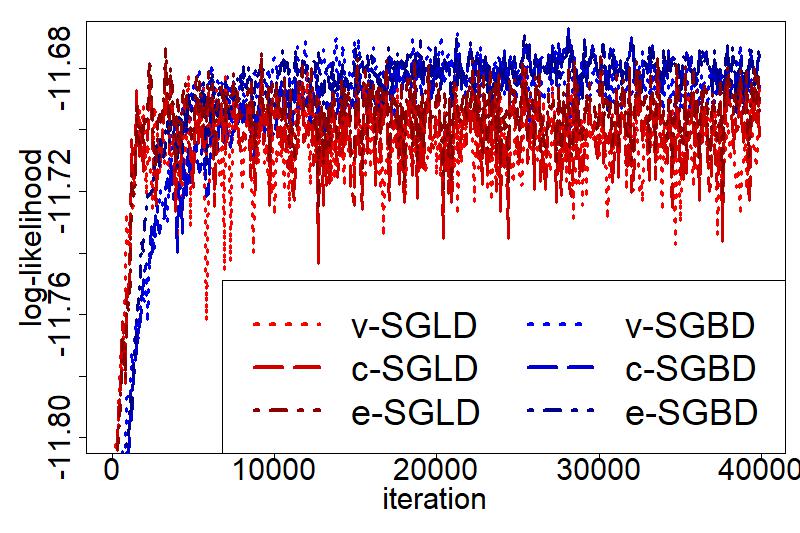}
         \caption{Samples log-likelihood}
         \label{ica_log_lik_plot}
     \end{subfigure}
     \hfill
     \begin{subfigure}[b]{0.48  \textwidth}
         \centering
         \includegraphics[width=\textwidth]{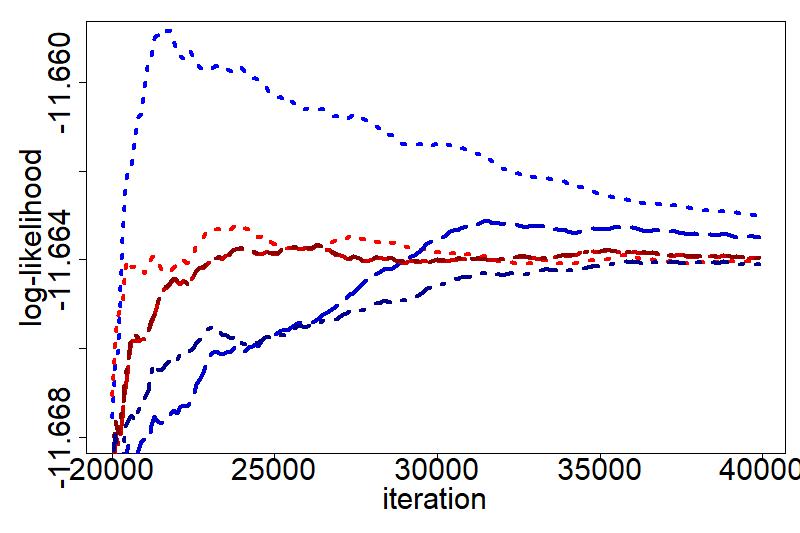}
         \caption{MCMC log-likelihood}
         \label{ica_log_lik_mcmc_plot}
     \end{subfigure}
     \vspace{.3in}
        \caption[ICA log likelihood]{ICA: log likelihood on the MEG dataset. Log likelihood produced by each sample (left) and by the MCMC estimates (right) on held-out data. Red refers to SGLD and blue to SGBD. For both algorithms, the vanilla (lighter dotted lines), corrected (medium scale dashed lines) and extreme (darker dotted-dashed lines) versions are displayed.}
        \label{ica_plot_1}
\end{figure}
We test the performance on the unseen test set both by computing the log-likelihood produced by each sample, as well as using the one produced by the ergodic averages of $W$, see the supplement for explicit definitions.  
Results are reported in Figure \ref{ica_plot_1}.
In this example optimal performances were obtained with very small stepsizes, both for SGLD and SGBD, 
and in such setting the difference between the two algorithms is more limited.

Overall, the c-SGBD estimates produce higher levels of log-likelihood on unseen data.

\section{CONCLUSION}\label{sec:conclusion}
In this paper, we extended the Barker proposal to the stochastic gradient setting, leading to the SGBD algorithm. 
We studied the bias induced by stochastic gradient noise in the Barker proposal and develop strategies to remove it, for small noise levels, or minimize it, for larger ones. 

 We then compared SGBD to SGLD numerically on simulated and real datasets. Results suggest that, while the two algorithms have similar performances for small step-sizes, SGBD is more robust to hyperparameters choice (thus potentially allowing for larger values to be used in practice) and to heterogeneity in the target gradients (arising from e.g.\ skewness or ill-conditioning). 
 
 Overall, SGBD represents a valid alternative to SGLD with minimal algorithmic change and appealing robustness.

There are many directions for future research, including: adding momentum to improve efficiency (similarly to, e.g., SGHMC \citep{sg_hamiltonian}); developing adaptive variants that optimally tunes the step-size across iterations; characterizing more explicitly how the invariant distribution is affected by the choice of the step-size; directly studying impact on predictive accuracy;  
 further exploring the connection with optimization schemes as well as the use of SGBD for optimization purposes.

\subsubsection*{Acknowledgements}

GZ acknowledges support from the European Research Council (ERC),
through StG ``PrSc-HDBayLe'' grant ID 101076564. LM was partially supported by the National Institutes of Health, grant ID R01ES035625.

\bibliography{References}

\begin{thebibliography}{31}
\providecommand{\natexlab}[1]{#1}
\providecommand{\url}[1]{\texttt{#1}}
\expandafter\ifx\csname urlstyle\endcsname\relax
  \providecommand{\doi}[1]{doi: #1}\else
  \providecommand{\doi}{doi: \begingroup \urlstyle{rm}\Url}\fi

\bibitem[Amari et~al.(1995)Amari, Cichocki, and Yang]{ica_95}
Shun-ichi Amari, Andrzej Cichocki, and Howard Yang.
\newblock A new learning algorithm for blind signal separation.
\newblock In D.~Touretzky, M.C. Mozer, and M.~Hasselmo, editors, \emph{Advances in Neural Information Processing Systems}, volume~8. MIT Press, 1995.

\bibitem[Azzalini(2013)]{azzalini2013skew}
Adelchi Azzalini.
\newblock \emph{The skew-normal and related families}, volume~3.
\newblock Cambridge University Press, 2013.

\bibitem[Bardenet et~al.(2014)Bardenet, Doucet, and Holmes]{bardenet14}
Rémi Bardenet, Arnaud Doucet, and Chris Holmes.
\newblock Towards scaling up markov chain monte carlo: an adaptive subsampling approach.
\newblock In Eric~P. Xing and Tony Jebara, editors, \emph{Proceedings of the 31st International Conference on Machine Learning}, volume~32 of \emph{Proceedings of Machine Learning Research}, pages 405--413, Bejing, China, 22--24 Jun 2014. PMLR.

\bibitem[Betancourt(2015)]{incompatibility_hmc}
Michael Betancourt.
\newblock The fundamental incompatibility of scalable hamiltonian monte carlo and naive data subsampling.
\newblock In Francis Bach and David Blei, editors, \emph{Proceedings of the 32nd International Conference on Machine Learning}, volume~37 of \emph{Proceedings of Machine Learning Research}, pages 533--540, Lille, France, 07--09 Jul 2015. PMLR.

\bibitem[Bowling et~al.(2009)Bowling, Khasawneh, Kaewkuekool, and Cho]{logistic_normal_approx_09}
Shannon Bowling, Mohammad Khasawneh, Sittichai Kaewkuekool, and Byung Cho.
\newblock A logistic approximation to the cumulative normal distribution.
\newblock \emph{Journal of Industrial Engineering and Management}, 2, 07 2009.

\bibitem[Brosse et~al.(2018)Brosse, Durmus, and Moulines]{brosse_et_al18}
Nicolas Brosse, Alain Durmus, and Eric Moulines.
\newblock The promises and pitfalls of stochastic gradient langevin dynamics.
\newblock \emph{Advances in Neural Information Processing Systems}, 31, 2018.

\bibitem[Carpenter et~al.(2017)Carpenter, Gelman, Hoffman, Lee, Goodrich, Betancourt, Brubaker, Guo, Li, and Riddell]{stan}
Bob Carpenter, Andrew Gelman, Matthew Hoffman, Daniel Lee, Ben Goodrich, Michael Betancourt, Marcus Brubaker, Jiqiang Guo, Peter Li, and Allen Riddell.
\newblock Stan: {A} probabilistic programming language.
\newblock \emph{Journal of Statistical Software}, 76\penalty0 (1), 2017.

\bibitem[Chen et~al.(2014)Chen, Fox, and Guestrin]{sg_hamiltonian}
Tianqi Chen, Emily Fox, and Carlos Guestrin.
\newblock Stochastic gradient hamiltonian monte carlo.
\newblock In Eric~P. Xing and Tony Jebara, editors, \emph{Proceedings of the 31st International Conference on Machine Learning}, volume~32 of \emph{Proceedings of Machine Learning Research}, pages 1683--1691, Bejing, China, Jun 2014. PMLR.

\bibitem[Coullon et~al.(2023)Coullon, South, and Nemeth]{mamba}
Jeremie Coullon, Leah South, and Christopher Nemeth.
\newblock Efficient and generalizable tuning strategies for stochastic gradient mcmc.
\newblock \emph{Statistics and Computing}, 33, 04 2023.

\bibitem[Ding et~al.(2014)Ding, Fang, Babbush, Chen, Skeel, and Neven]{sgnht}
Nan Ding, Youhan Fang, Ryan Babbush, Changyou Chen, Robert~D Skeel, and Hartmut Neven.
\newblock Bayesian sampling using stochastic gradient thermostats.
\newblock In Z.~Ghahramani, M.~Welling, C.~Cortes, N.~Lawrence, and K.Q. Weinberger, editors, \emph{Advances in Neural Information Processing Systems}, volume~27. Curran Associates, Inc., 2014.

\bibitem[Duane et~al.(1987)Duane, Kennedy, Pendleton, and Roweth]{duane_et_al_87}
Simon Duane, A.D. Kennedy, Brian~J. Pendleton, and Duncan Roweth.
\newblock Hybrid monte carlo.
\newblock \emph{Physics Letters B}, 195\penalty0 (2):\penalty0 216--222, 1987.

\bibitem[Duchi et~al.(2011)Duchi, Hazan, and Singer]{adagrad}
John Duchi, Elad Hazan, and Yoram Singer.
\newblock Adaptive subgradient methods for online learning and stochastic optimization.
\newblock \emph{Journal of Machine Learning Research}, 12\penalty0 (61):\penalty0 2121--2159, 2011.

\bibitem[Hird et~al.(2020)Hird, Livingstone, and Zanella]{barker_fresh_take}
Max Hird, Samuel Livingstone, and Giacomo Zanella.
\newblock A fresh take on 'barker dynamics' for mcmc.
\newblock In \emph{Monte Carlo and Quasi-Monte Carlo Methods}, 2020.

\bibitem[Jitkrittum et~al.(2017)Jitkrittum, Xu, Szabo, Fukumizu, and Gretton]{fssd}
Wittawat Jitkrittum, Wenkai Xu, Zoltan Szabo, Kenji Fukumizu, and Arthur Gretton.
\newblock A linear-time kernel goodness-of-fit test.
\newblock In I.~Guyon, U.~Von Luxburg, S.~Bengio, H.~Wallach, R.~Fergus, S.~Vishwanathan, and R.~Garnett, editors, \emph{Advances in Neural Information Processing Systems}, volume~30. Curran Associates, Inc., 2017.

\bibitem[Korattikara et~al.(2014)Korattikara, Chen, and Welling]{austerity_mcmc}
Anoop Korattikara, Yutian Chen, and Max Welling.
\newblock Austerity in mcmc land: Cutting the metropolis-hastings budget.
\newblock In Eric~P. Xing and Tony Jebara, editors, \emph{Proceedings of the 31st International Conference on Machine Learning}, volume~32 of \emph{Proceedings of Machine Learning Research}, pages 181--189, Bejing, China, Jun 2014. PMLR.

\bibitem[Livingstone and Zanella(2022)]{barker}
Samuel Livingstone and Giacomo Zanella.
\newblock The barker proposal: Combining robustness and efficiency in gradient‐based mcmc.
\newblock \emph{Journal of the Royal Statistical Society: Series B (Statistical Methodology)}, 2022.

\bibitem[Lu et~al.(2017)Lu, Perrone, Hasenclever, Teh, and Vollmer]{relativistic_mc_17}
Xiaoyu Lu, Valerio Perrone, Leonard Hasenclever, Yee~Whye Teh, and Sebastian Vollmer.
\newblock {Relativistic Monte Carlo }.
\newblock In Aarti Singh and Jerry Zhu, editors, \emph{Proceedings of the 20th International Conference on Artificial Intelligence and Statistics}, volume~54 of \emph{Proceedings of Machine Learning Research}, pages 1236--1245. PMLR, 20--22 Apr 2017.

\bibitem[Ma et~al.(2015)Ma, Chen, and Fox]{ma_et_al_15}
Yi-An Ma, Tianqi Chen, and Emily~B. Fox.
\newblock A complete recipe for stochastic gradient mcmc.
\newblock In \emph{Proceedings of the 28th International Conference on Neural Information Processing Systems - Volume 2}, NIPS'15, page 2917–2925, Cambridge, MA, USA, 2015. MIT Press.

\bibitem[Neal(2012)]{neal11}
Radford Neal.
\newblock Mcmc using hamiltonian dynamics.
\newblock \emph{Handbook of Markov Chain Monte Carlo}, 2012.

\bibitem[Nemeth and Fearnhead(2021)]{nemeth2021sgmc}
Christopher Nemeth and Paul Fearnhead.
\newblock Stochastic gradient markov chain monte carlo.
\newblock \emph{Journal of the American Statistical Association}, 116\penalty0 (533):\penalty0 433--450, 2021.

\bibitem[Power and Goldman(2019)]{power2019accelerated}
Samuel Power and Jacob~Vorstrup Goldman.
\newblock Accelerated sampling on discrete spaces with non-reversible markov processes.
\newblock \emph{arXiv preprint arXiv:1912.04681}, 2019.

\bibitem[Robbins and Monro(1951)]{robbins51}
Herbert Robbins and Sutton Monro.
\newblock A stochastic approximation method.
\newblock \emph{The Annals of Mathematical Statistics}, 22\penalty0 (3):\penalty0 400 -- 407, 1951.

\bibitem[Roberts and Rosenthal(1995)]{roberts95}
Gareth~O. Roberts and Jeffrey~S. Rosenthal.
\newblock Optimal scaling of discrete approximations to langevin diffusions.
\newblock \emph{Journal of the Royal Statistical Society: Series B (Statistical Methodology)}, 60\penalty0 (1):\penalty0 255--268, 1995.

\bibitem[Roberts and Tweedie(1996)]{roberts1996exponential}
Gareth~O Roberts and Richard~L Tweedie.
\newblock Exponential convergence of langevin distributions and their discrete approximations.
\newblock \emph{Bernoulli}, pages 341--363, 1996.

\bibitem[Salakhutdinov and Mnih(2008)]{bpmf}
Ruslan Salakhutdinov and Andriy Mnih.
\newblock Bayesian probabilistic matrix factorization using markov chain monte carlo.
\newblock In \emph{Proceedings of the 25th International Conference on Machine Learning}, ICML '08, page 880–887, New York, NY, USA, 2008. Association for Computing Machinery.

\bibitem[Sun et~al.(2023)Sun, Dai, Dai, Zhou, and Schuurmans]{sun2023discrete}
Haoran Sun, Hanjun Dai, Bo~Dai, Haomin Zhou, and Dale Schuurmans.
\newblock Discrete langevin samplers via wasserstein gradient flow.
\newblock In \emph{International Conference on Artificial Intelligence and Statistics}, pages 6290--6313. PMLR, 2023.

\bibitem[Teh et~al.(2016)Teh, Thiery, and Vollmer]{teh_et_al15}
Yee~Whye Teh, Alexandre~H. Thiery, and Sebastian~J. Vollmer.
\newblock Consistency and fluctuations for stochastic gradient langevin dynamics.
\newblock \emph{Journal of Machine Learning Research,}, 17\penalty0 (1):\penalty0 193–225, 2016.

\bibitem[Vogrinc et~al.(2022)Vogrinc, Livingstone, and Zanella]{optimal_design_barker}
Jure Vogrinc, Samuel Livingstone, and Giacomo Zanella.
\newblock {Optimal design of the Barker proposal and other locally balanced Metropolis–Hastings algorithms}.
\newblock \emph{Biometrika}, 10 2022.
\newblock ISSN 1464-3510.

\bibitem[Vollmer et~al.(2016)Vollmer, Zygalakis, and Teh]{vollmer_et_al_16}
Sebastian~J. Vollmer, Konstantinos~C. Zygalakis, and Yee~Whye Teh.
\newblock Exploration of the (non-)asymptotic bias and variance of stochastic gradient langevin dynamics.
\newblock \emph{Journal of Machine Learning Research}, 17\penalty0 (159):\penalty0 1--48, 2016.

\bibitem[Welling and Teh(2011)]{sgld}
Max Welling and Yee~Whye Teh.
\newblock Bayesian learning via stochastic gradient langevin dynamics.
\newblock In \emph{Proceedings of the 28th International Conference on International Conference on Machine Learning}, ICML'11, page 681–688, Madison, WI, USA, 2011. Omnipress.

\bibitem[Zanella(2020)]{informed_proposals}
Giacomo Zanella.
\newblock Informed proposals for local mcmc in discrete spaces.
\newblock \emph{Journal of the American Statistical Association}, 115\penalty0 (530):\penalty0 852--865, 2020.

\end{thebibliography}

\clearpage
\onecolumn
\aistatstitle{Supplementary Materials}
\appendix
\section{ALGORITHMS}
In this section we provide the pseudocode for the e-SGBD, v-SGLD and c-SGLD algorithms described in Section \ref{subsec:e_sgbd} of the paper.

\begin{algorithm}
\SetAlgoLined
\SetKwInOut{Input}{Input}   \Input{$\theta^{(0)}\in \mathbb R^d, \sigma >0$}
\For{t =1,\dots, T} {
 Draw $\mathcal{S}_n \subset \{1, \dots, N\}$ uniformly at random\;
 \For{j=1, \dots, d \tcp*{Can be parallelized}}{
 Compute $\hat{\partial}_j g(\theta^{(t-1)})$ as in \eqref{eq:stochastic_gradient}\;
 Draw $w_j^{(t)} \sim N(\sigma, (0.1\sigma)^2)$\;
 Set $b_j^{(t)} = 1$ if  $w_j^{(t)}\hat{\partial}_j g(\theta^{(t-1)}) >0$ otherwise $b_j^{(t)} = -1$\;
 Update $\theta^{(t)} \gets \theta^{(t-1)} +b_j^{(t)}w_j^{(t)}$ \;
}}
 \caption{Extreme Stochastic Gradient Barker Dynamics (e-SGBD)}\label{alg:e_sgbd}
\end{algorithm}\

\begin{algorithm}
\SetAlgoLined
\SetKwInOut{Input}{Input}   \Input{$\theta^{(0)}\in \mathbb R^d, \sigma >0$}
 \For{t =1,\dots, T} {
 Draw $\mathcal{S}_n \subset \{1, \dots, N\}$ uniformly at random\;
 \For{j=1, \dots, d \tcp*{Can be parallelized}}{
 Compute $\hat{\partial}_j g(\theta^{(t-1)})$ as in \eqref{eq:stochastic_gradient}\;
 Draw $\epsilon_j^{(t)} \sim N(0, \sigma^2)$\;
 Update $\theta^{(t)} \gets \theta^{(t-1)} + \frac{\sigma^2}{2}\hat{\partial}_j g(\theta^{(t-1)}) + \epsilon_j^{(t)}$\;
 } 
 }
 \caption{Stochastic Gradient Langevin Dynamics (v-SGLD)}
 \label{SGLD_algo}
\end{algorithm}\

\begin{algorithm}
\SetAlgoLined
\SetKwInOut{Input}{Input}   \Input{$\theta^{(0)}\in \mathbb R^d, \sigma>0, \beta\in(0,1), \{\hat{\tau}_{j}^{(0)}\}_{j=1,\dots,d}$}
 \For{t =1,\dots, T} {
 Draw $\mathcal{S}_n \subset \{1, \dots, N\}$ uniformly at random\;
  \For{t =1,\dots, T \tcp*{Can be parallelized}} {
 Compute $\hat{\partial}_j g(\theta^{(t-1)})$ as in \eqref{eq:stochastic_gradient}\;
 Update $\hat\tau_j^{(t)} \gets (1-\beta)\hat\tau_j^{(t-1)} + \beta \sqrt{ \frac{1}{n-1}\sum_{i\in\mathcal{S}_n}(\partial_j g_i(\theta^{(t-1)}) - \frac{1}{n}\sum_{i\in\mathcal{S}_n}\partial_j g_i(\theta^{(t-1)}))^2}$\;
 Draw $\epsilon_j^{(t)} \sim N\left(0, \max\left(0, \sigma^2 - \frac{\hat{\tau}^{(t)}2}{4}\sigma^4\right) \right)$\;
 Update $\theta^{(t)} \gets \theta^{(t-1)} + \frac{\sigma^2}{2}\hat{\partial}_j g(\theta^{(t-1)}) + \epsilon_j^{(t)}$\;
 }  }

 \caption{Corrected Stochastic Gradient Langevin Dynamics (c-SGLD)}
 \label{cSGLD_algo}
\end{algorithm}\

\section{PROOFS AND ADDITIONAL LEMMAS}\label{subsec:proofs}

\subsection{Proof of Proposition \ref{prop:bias_p_hat}}\label{proof:prop:bias_p_hat}
We first state two auxiliary lemmas.

\begin{lemma}\label{lemma:1}
Let $(\delta, z) \in \mathbb R^2$ and $\eta \neq 0$.
Then 
$\frac{1}{2}\left(p(\delta - \eta, z) + p(\delta+ \eta, z) \right)\leq  p(\delta, z)$ if and only if $z \delta \geq 0$, with equality only if $z\delta = 0$.

\end{lemma}
\begin{proof}
We have
\begin{equation}
\label{eq:lemma_1_a}
    2 p(\delta, z) - \left(p(\delta - \eta, z) + p(\delta + \eta, z)\right) = \left(1 + e^{-z\delta}\right)^{-1} - \left(\left(1 + e^{-z\delta + z\eta}\right)^{-1} + \left(1 + e^{-z\delta - z\eta}\right)^{-1}\right)\,.
\end{equation}
Define $c = e^{-z\delta}$. 
Expanding the right hand side of \eqref{eq:lemma_1_a}, we obtain
\begin{equation}
\label{eq:lemma_1}
    \begin{aligned}
      &\left(1 + c\right)^{-1} - \left(\left(1 + c e^{+ z\eta}\right)^{-1} + \left(1 + c e^{- z\eta}\right)^{-1}\right)
           &=  \frac{c (1-c) \left(\left(e^{-z \eta} + e^{z \eta}\right) -2\right)}{(1 + c)\left(1+ c e^{- z \eta}\right)\left(1+ c e^{+ z \eta}\right)}
    \end{aligned}
\end{equation}
    Note that the denominator on the right hand side of \eqref{eq:lemma_1} is strictly positive and $\left(e^{-z \eta} + e^{z \eta}\right) -2 \geq 0$ by Jensen's inequality with strict inequality if $z \neq 0$. The result follows by noting that $c >0$ always, and $1-c\geq 0$ if and only if $z\delta \geq 0$, with strict inequality if $z\delta >0$.
\end{proof}

\begin{lemma}\label{lemma:2}
Let $(\delta, z) \in \mathbb R^2$ and $\eta \neq 0$.
Then 
$\frac{1}{2}\left(p(\delta - \eta, z) + p(\delta + \eta, z)\right) \geq  0.5$ if and only if $z \delta \geq 0$, with equality only if $z\delta = 0$.

\end{lemma}

\begin{proof}
Define $c:= e^{-z\delta}$ and consider the following:
     \begin{equation}
         \begin{aligned}
            \frac{1}{2}\left(p(\delta+\eta, z) + p(\delta-\eta, z) \right) &= \frac{1}{2}  \left(\left(1 +c e^{-z\eta}\right)^{-1} + \left(1+c e^{z\eta}\right)^{-1} \right) \\
            &= \frac{2 + c \left(e^{- z \eta} + e^{+z \eta}\right)}{2\left(1+  c e^{- z \eta}\right)\left(1+  c e^{+ z \eta}\right) }\\
            &= \frac{2 + c \left(e^{- z \eta} + e^{+z \eta}\right)}{2 + 2 c \left(e^{- z \eta} + e^{+z \eta}\right) + c ^2}
         \end{aligned}
     \end{equation}
     The result follows by noting that $c ^2 \leq 1$ if and only if $ z\delta\geq 0$, with strict inequality if $z\delta>0$, and $\left(e^{-z \eta} + e^{z \eta}\right) -2 \geq 0$ by Jensen's inequality with strict inequality if $z \neq 0$.
\end{proof}
We can now prove Proposition \ref{prop:bias_p_hat}.
\begin{proof}[Proof of Proposition \ref{prop:bias_p_hat}]

Consider first the case $z\partial_j g(\theta)>0$ and $\tau_\theta>0$. 
Denoting  the distribution of $\eta_\theta$ by $P_{\eta_\theta}$, we have
\begin{equation}
    \begin{aligned}
  \mathbb{E}\left[ p \left(\partial_j \hat g(\theta),  z\right)\right]&\overset{(I)}{=} \int_{0}^{+\infty} \left( p\left(\partial_j g(\theta)+\eta_\theta, z\right) +  p\left(\partial_j g(\theta)-\eta_\theta, z\right) \right)
  dP_{\eta_\theta}(\eta_\theta) 
  \\& =
  \int_0^{+\infty} \left(\left(1+ e^{- z (\partial_j g(\theta) +\eta_\theta)}\right)^{-1} + \left(1+ e^{- z (\partial_j g(\theta) -\eta_\theta)}\right)^{-1} \right) 
  dP_{\eta_\theta}(\eta_\theta)
  \\ &
  \overset{(II)}{\leq}
   \int_0^{+\infty} 2\left(1+ e^{- z (\partial_j g(\theta))}\right)^{-1}
  dP_{\eta_\theta}(\eta_\theta)
   \\
   &\overset{(III)}{=}\int_{-\infty}^{+\infty} \left(1+ e^{- z (\partial_j g(\theta))}\right)^{-1}
     dP_{\eta_\theta}(\eta_\theta)
= 
p(\partial_j g(\theta), z),
    \label{eq:E_p_hat}
\end{aligned}
\end{equation}
where $(I)$ and $(III)$ follow from the symmetry assumption in Condition \ref{cond:symmetry}, and $(II)$ follows from Lemma \ref{lemma:1}
Moreover, $z\partial_j g(\theta)>0$ implies $p(\partial_j g(\theta), z)>0.5$ and $\frac{1}{2}\left(p(\partial_j g(\theta)+\eta_\theta, z) + p(\partial_j g(\theta)-\eta_\theta, z) \right)> 0.5$ for all values $\eta_\theta \neq 0$ by Lemma \ref{lemma:2}.
It follows, again using the symmetry assumption in Condition \ref{cond:symmetry}, that
\begin{equation}
  0.5 < \mathbb E\left[\hat{p}\left(\hat{\partial}_j g(\theta), z\right) \right] < p(\partial_j g(\theta), z),
\end{equation}
Using the same argument, it is easy to show that the reverse inequalities hold if $z\partial_j g(\theta)<0$, leading to $\left|p\left(\partial_j g(\theta), z\right) - 0.5\right| > \left|\mathbb E\left[ p\left(\hat{\partial}_j g(\theta), z\right)\right]- 0.5\right|$ as desired.
The argument for the case $z\partial_j g(\theta)<0$ is analogous.
Finally, if $z\partial_j g(\theta)=0$ or $\tau_\theta = 0$, we have $\mathbb{E}\left[ p \left(\partial_j \hat g(\theta),  z\right)\right]=p(\partial_j g(\theta), z)=0.5$. 
\end{proof}

\subsection{Proof of Proposition \ref{prop:p_hat_biased}}\label{proof:prop:p_hat_biased}
\begin{proof}
Let $F(x)=(1+\exp(-x))^{-1}$ be the CDF of the logistic distribution and $\Phi$ the one of the standard Normal distribution.
Then we have $\left|F(x) - \Phi\left(x/1.702\right) \right|< 0.0095$ for all $x \in \mathbb{R}$, see e.g.\ \citet[Section~3.2]{logistic_normal_approx_09}. 
By definition of $p$, this implies
\begin{align}\label{ineq:logistic_cdf_bound}
   \left|p(\delta,z) - \Phi\left(\frac{z\delta}{1.702}\right) \right| &< 0.0095  &\delta,z \in \mathbb R
\end{align}
and, as a result,
\begin{align}\label{eq_E_p_hat_normal}
&  \left| 
   \mathbb{E}
   \left[p\left(\hat{\partial}_j g(\theta), z\right)
   -
   \Phi\left(\frac{z \hat{\partial}_j g(\theta)}{1.702}\right)\right]
     \right|
  < 0.0095.
  \end{align}
Hence, $ \left| \mathbb E\left[p\left(\hat{\partial}_j g, z\right)\right] - \Phi\left(\frac{z\partial_j g(\theta)}{\sqrt{1.702^2+z^2\tau_\theta^2)}}\right)\right| < 0.0095$ noting that
\begin{equation}
    \mathbb E \left[\Phi\left(\frac{z \hat{\partial}_j g(\theta)}{1.702}\right)\right] = \Phi\left(\frac{z \partial_j g(\theta)}{\sqrt{1.702^2+z^2\tau_\theta^2}}\right)\,,
\end{equation}
which follows by  $\hat{\partial}_j g(\theta) \sim \mathcal{N}(\partial_j g(\theta), \tau_\theta^2)$ and standard properties of the normal distribution.
The result in Equation \eqref{eq:approx_expectation_p_hat} follows using the bound \eqref{ineq:logistic_cdf_bound} with an application of the triangle inequality.
\end{proof}

\subsection{Proof of Corollary \ref{prop:unbiased_p_tilde}}\label{proof:prop:unbiased_p_tilde}
\begin{proof}
Consider first the case when $\tau_\theta < 1.702/|z|$.
Then, replacing $z$ with $z\alpha$ in Proposition \ref{prop:p_hat_biased}, we obtain
\begin{equation}
\left|\mathbb{E}\left[p\left(\hat{\partial}_j g, \alpha z\right)\right] - p\left(\frac{1.702 \alpha }{\sqrt{1.702^2+z^2\tau_\theta^2\alpha^2}} \partial_j g(\theta), z\right)\right| < 0.019. 
\label{E_eq:p_tilde_alpha}
\end{equation} 
Taking $\alpha = \frac{1.702}{ \sqrt{1.702^2-\tau_\theta^2z^2}}$, we have $p\left(\hat{\partial}_j g, \alpha z\right) = \tilde{p}\left(\hat{\partial}_j g, z\right)$  and the second term in the absolute value on the left hand side of \eqref{E_eq:p_tilde_alpha} simplifies to $p(\partial_j g(\theta), z)$ leading to the desired result. 

Consider now the case when $\bar \tau(\partial_j g(\theta), z) \geq 1.702/|z|$ and $\tau_\theta \in \left[1.702/|z|, \bar \tau(\partial_j g(\theta), z) \right]$. In this case, we have $\tilde p\left(\hat{\partial}_j g(\theta),  z\right) = \mathbf{1}(\hat{\partial}_j g(\theta)z>0)$.
We first consider the case when $z\partial_j g(\theta)>0$. 
Under Condition \ref{cond:normality}, we have
\begin{equation}
\label{eq:E_p_tilde_above_bp}
    \mathbb{E}\left[\tilde p\left(\hat{\partial}_j g,  z\right)\right] = \Phi\left(|\partial_jg(\theta)|/\tau_\theta\right).
\end{equation}
The right hand side of \eqref{eq:E_p_tilde_above_bp} is a strictly increasing function  of $\tau_\theta$, hence it can be lower and upper bounded by 
\begin{equation}
\label{eq:E_p_tilde_above_bp_bound}
    \Phi\left(|\partial_jg(\theta)|/\bar \tau(\partial_j g(\theta), z) \right)
\leq \Phi\left(|\partial_jg(\theta)|/\tau_\theta\right) \leq \Phi\left(\partial_jg(\theta)z/1.702\right).
\end{equation}
Define $D(x) : =\Phi(x/1.702) - (1 + e^{-x})^{-1}$ and expand the upper bound in \eqref{eq:E_p_tilde_above_bp_bound} as
\begin{equation}
\begin{aligned}
    \Phi\left(\partial_jg(\theta)z/1.702\right) &= \left(1 + e^{-\partial_jg(\theta)z}\right) - D(\partial_jg(\theta)z)\\
    &\leq p\left(\partial_jg(\theta), z\right) + 0.0095.
\end{aligned}
\end{equation}
Moreover, note that, when $z\partial_jg(\theta)>0 $, the left hand side in \eqref{eq:E_p_tilde_above_bp} is equal to $p\left(\partial_jg(\theta), z\right)$ by the definition of $\tau(\partial_j g(\theta), z)$, i.e.
where the inequality follows from the definition of $p$ and $|D(x)| < 0.0095$ for all $x \in \mathbb R$ \cite[Section~3.2]{logistic_normal_approx_09}.
\begin{equation}
\label{eq:E_p_tilde_above_bp_lo_bound}
\Phi\left(\partial_jg(\theta)/\bar \tau(\partial_j g(\theta), z) \right) = p\left(\partial_jg(\theta), z\right).
\end{equation}

Thus, combining all of the above, we get 
\begin{equation}
  p\left(\partial_jg(\theta), z\right)  \leq \mathbb{E}\left[\tilde p\left(\hat{\partial}_j g(\theta),  z\right)\right] \leq p\left(\partial_jg(\theta), z\right) + 0.0095.
\end{equation}

When $\partial_jg(\theta)z\leq 0$, it is easy to show that the reverse inqualities hold, leading to the following bound:
\begin{equation}
  p\left(\partial_jg(\theta), z\right) -0.0095 \leq \mathbb{E}\left[\tilde p\left(\hat{\partial}_j g,  z\right)\right] \leq p\left(\partial_jg(\theta), z\right).
\end{equation}
\end{proof}

\subsection{Proof of Proposition \ref{prop:breaking_point} }
\begin{proof}
To prove the statement, consider the expectation of the \textit{extreme}-estimator defined in Section \ref{subsec:e_sgbd} of the paper, $\bar{p}$, 
    \begin{equation}
    \label{eq:expectation_p_bar_appendix}
    \mathbb E\left[\bar{p}\left(\hat{\partial}_jg(\theta), z\right)\right] = \Phi\left(\tau_\theta^{-1}\left|\partial_j g(\theta)\right|sgn(z\partial_j g(\theta)) \right).
    \end{equation}
    Note that the definition $\bar \tau(\partial_j g(\theta), z)$ in \eqref{eq:breaking_point} implies that
\begin{equation}
     \Phi\left(\left|\partial_j g(\theta)\right|sgn(z\partial_j g(\theta)) /\tau(\partial_j g(\theta), z) \right) = p\left(\partial_jg(\theta), z\right).
    \label{eq:implicit_breaking_point_1}
\end{equation}
Consider the case when $\tau_\theta > \tau^*$. Since $\bar \tau( \delta, z)$ is a continuous function of $ \delta$, there exists $\bar  \delta \in \mathbb R$ such that $\tau^* < \bar \tau(\bar  \delta, z) < \tau_\theta$. 
 
Notice that the right hand side of \eqref{eq:expectation_p_bar_appendix} is decreasing in $\tau_\theta$ if $z\partial_j g(\theta) >0$ and increasing otherwise. 
Hence, for a value of the gradient $\bar  \delta$ and noise standard deviation $\tau_\theta > \bar \tau(\bar \delta, z)$, we obtain 
\begin{equation}
   \left|p\left(\bar  \delta, z\right) - 0.5\right| < \left|\mathbb E\left[\bar{p}\left(\hat{\partial}_jg(\theta), z\right)\right] - 0.5\right|.
\end{equation}

By Proposition \ref{prop:optimality_p_bar}, we know that, for any other symmetric estimator $\hat p$ satisfying \eqref{eq:symmetric_estimator}, we have
\begin{equation}
   \left|\mathbb E\left[\bar{p}\left(\hat{\partial}_jg(\theta), z\right)\right] - 0.5\right| <  \left|\mathbb E\left[\hat{p}\left(\hat{\partial}_jg(\theta), z\right)\right] - 0.5\right| 
\end{equation}
Hence, no unbiased symmetric estimator exists for $\tau_\theta > \tau^*$. 

Finally, note that the absence of any unbiased symmetric estimator implies the absence of any unbiased estimator, since any unbiased estimator needs to be symmetric.
Indeed, assume by contradiction that $\hat p(\hat \delta, z)$ is unbiased for $p(\delta, z)$ when $\hat \delta \sim \mathcal N( \delta, \tau_\theta)$ but it is not symmetric, i.e. $\hat p(\hat \delta, z) + \hat p(-\hat \delta, z) \neq 1$ for some $\hat \delta \in \mathbb R$. Consider $z$ fixed and define $h(\hat \delta) = \hat p(\hat \delta, z) + \hat p(-\hat \delta, z) -1$. Unbiasedness imply $\mathbb E\left[\hat p(\hat \delta, z) + \hat p(-\hat \delta, z)\right] =p( \delta, z) + p(- \delta, z)= 1$, or equivalently
 $\mathbb E\left[h(\hat \delta)\right] = 0$ with $\hat \delta \sim \mathcal N( \delta, \tau_\theta^2)$, for all $\delta \in \mathbb R$. Since $\hat \delta$ is a complete sufficient statistics for $\hat{\delta} \sim \mathcal N( \delta, \tau_\theta^2)$, the ladder condition implies  $h(\hat \delta) \equiv 0$, determining a contradiction.
\end{proof}

\subsection{Proof of Proposition \ref{prop:optimality_p_bar}}

\begin{proof}
The case $z=0$ is trivial. 
We prove the result when $z>0$, the case $z<0$ is analogous. 
For any symmetric estimator $\hat p$ satisfying \eqref{eq:symmetric_estimator}, we have
\begin{equation}
    \begin{aligned}
   \mathbb E\left[\hat p(\hat{\partial}_j g(\theta), z)\right] &= \int_{-\infty}^{+\infty} \hat p(\partial_j g(\theta) + \eta_\theta, z) f_\theta(\eta_\theta) d\eta_\theta 
   \\
   &\overset{(I)}{=} 
   \int_{-\infty}^{+\infty} \hat p(\varepsilon, z) f_\theta(\varepsilon -    \partial_j g(\theta) ) d\varepsilon 
   \\
    &= 
    \int_{0}^{+\infty} \hat p(\varepsilon, z) f_\theta(\varepsilon -    \partial_j g(\theta) ) d\varepsilon
    +
    \int_{-\infty}^{0}\hat p(\varepsilon, z) f_\theta(\varepsilon -    \partial_j g(\theta) ) d\varepsilon
    \\
    & \overset{(II)}{=}
    \int_{0}^{+\infty} \hat p(\varepsilon, z) f_\theta(\varepsilon -    \partial_j g(\theta) ) d\varepsilon
    +
    \int_{0}^{+\infty}\hat p(-\epsilon, z) f_\theta(-\epsilon -    \partial_j g(\theta) ) d\epsilon
    \\
    & \overset{(III)}{=} 
    \int_{0}^{+\infty} \hat p(\varepsilon, z) f_\theta(\varepsilon -    \partial_j g(\theta) ) d\varepsilon
    +
    \int_{0}^{+\infty}\hat p(-\epsilon, z) f_\theta(\epsilon +  \partial_j g(\theta) ) d\epsilon
\end{aligned} 
  \label{eq:expectation_p_doublehat}  
\end{equation}
where $(I)$ and ($II$) follows by changes of variables ($\varepsilon: = \eta_\theta + \partial_j g(\theta)$ and $\epsilon := - \varepsilon$, respectively), and $(III)$ follows by Conditions 1 and 3 which implies $f_\theta(x) = f_\theta(-x)$ for all $x \in \mathbb R$.

The expected value of the \textit{extreme}-estimator $\bar{p}$ is given by:
\begin{equation}
\begin{aligned}
    \mathbb E\left[\Bar{p}(\hat{\partial}_j g(\theta),z)\right] &= 
    \int_{-\infty}^{+\infty} \Bar{p}(\partial_j g(\theta) + \eta_\theta
, z)f_\theta(\eta_\theta)d\eta_\theta
\\
& =\int_{-\partial_j g(\theta)}^{+\infty} 1 f_\theta(\eta_\theta)d\eta_\theta 
\\
&  \overset{(I)}{=}\int_{0}^{+\infty} 
1 f_\theta(\epsilon -\partial_j g(\theta))d\epsilon
\\
&\overset{(II)}{=}
\int_{0}^{+\infty} \hat p\left(\epsilon,z\right) 
f_\theta(\epsilon -\partial_j g(\theta)) d\epsilon 
+
\int_{0}^{+\infty} \hat p\left(-\epsilon,z\right)f_\theta(\epsilon -\partial_j g(\theta)) d\epsilon
\end{aligned}
\label{eq:expectation_p_bar_2}
\end{equation}
where $(I)$ follows from the change of variable $\epsilon:= \eta_\theta + \partial_j g(\theta)$ and $(II)$ follows from \eqref{eq:symmetric_estimator}.
Thus
\begin{equation}
    \label{eq:diff_p_bar_p_sym}
 \mathbb E\left[\Bar{p}(\hat{\partial}_j g(\theta),z)\right]
 -
 \mathbb E\left[\hat p(\hat{\partial}_j g(\theta), z)\right]
 =
 \int_{0}^{+\infty} \hat p\left(-\epsilon,z\right)
 (f_\theta(\epsilon -\partial_j g(\theta))-f_\theta(\epsilon +  \partial_j g(\theta) )) d\epsilon.
\end{equation}
Moreover, note that by Conditions 1 and 3, if $\partial_j g(\theta) >0$ we have $f_\theta(\epsilon +  \partial_j g(\theta) )< f_\theta(\epsilon - \partial_j g(\theta) )$, while if $\partial_j g(\theta) <0$ we have $f_\theta(\epsilon +  \partial_j g(\theta) )> f_\theta(\epsilon - \partial_j g(\theta) )$, for every $\epsilon >0$.   
Hence, the right hand side in equation \eqref{eq:diff_p_bar_p_sym} is greater the than $0$ if $\partial_j g(\theta) >0$ 
and smaller then $0$ otherwise, with equality holding if and only if $\hat p\left(\epsilon,z\right)$ for all $\epsilon<0$, which holds if and only if $\hat p = \bar p$, if we exclude the trivial estimator $\hat p \equiv 1$.
\end{proof}

\subsection{Proof of Corollary \ref{cor:optimality_p_bar}}\begin{proof}
    By Proposition \ref{prop:optimality_p_bar}, we have $\mathbb E\left[\bar{p}\left(\hat{\partial}_j g(\theta), z\right)\right]> \mathbb E\left[\hat{p}\left(\hat{\partial}_j g(\theta), z\right)\right]$ if $\partial_j g(\theta) z>0$, while the reverse inequality holds if  $\partial_j g(\theta) z<0$.
    Moreover, by \eqref{eq:expectation_p_bar_appendix}, when $\tau_{\theta}>\bar\tau(\partial_j g(\theta),z)$ the expectation of $\bar{p}\left(\hat{\partial}_j g(\theta), z\right)$ is shrunk to $0.5$ compared to $p\left(\partial_j g(\theta), z\right)$. Thus, for $\partial_j g(\theta)z>0$, we have 
 \begin{equation} 0.5 <  \mathbb E\left[\hat{p}\left(\hat{\partial}_j g(\theta), z\right)\right] <  \mathbb E\left[\bar{p}\left(\hat{\partial}_j g(\theta), z\right) \right]  < p\left(\partial_j g(\theta), z\right),\end{equation} while the reverse inequalities hold when $\partial_j g(\theta)z<0$, proving the result. \end{proof}

\section{CONNECTION BETWEEN SGBD AND SGLD NOISE TOLERANCE}
One can compare the noise tolerance of SGBD with the one of SGLD. Consider the recursion of SGLD 
\begin{equation}
    \theta^{(t+1)} \gets \theta^{(t)} + \frac{\sigma^2}{2} \hat{\nabla}g(\theta^{(t)}) + z \quad z \sim \mathcal{N}(0, \sigma^2).
\end{equation} Under Condition \ref{cond:normality}, this recursion is equivalent to 
\begin{equation}
    \theta^{(t+1)} \gets \theta^{(t)} + \frac{\sigma^2}{2} \nabla g(\theta^{(t)}) + \tilde z \quad \tilde z \sim \mathcal{N}(0, \sigma^2 +(\sigma^2/2)^2 \tau_\theta^2).
\end{equation}
When $\tau_\theta \leq 2\sigma^{-1}$, it is possible to correct exactly the recursion reducing the variance of the artificial noise: namely, if $z \sim \mathcal{N}(0,\sigma^2 - (\sigma^2/2)^2 \tau_\theta^2 )$, we obtain back the exact unadjusted Langevin proposal, i.e.\ $\theta^{(t+1)} \sim \mathcal{N}( \theta^{(t)} + \frac{\sigma^2}{2} \nabla g(\theta^{(t)}) , \sigma^2)$.
On the contrary, if $\tau_\theta > 2\sigma^{-1}$, simple variance arguments show that there is no distribution for the noise $z$ (assuming it to be independent from $\eta_\theta$) such that the resulting proposal coincides with the exact unadjusted Langevin one. 
In this sense, the noise tolerance of SGLD under Condition \ref{cond:normality} is $\tau_{SGLD}^* = 2\sigma^{-1}$. 
Contrary to the one of SGBD, such value depends directly on the hyperparameter $\sigma$ and not on the sampled increment $z$. 
However, there are similarities with the one of SGBD identified in Section \ref{sec:noise_tolerance} of the paper. 
In particular, 
the size of the proposed increment under $\mu$ is of order 
$\sigma$, think e.g.\ at the choice $\mu = 0.5N(-\sigma,(0.1\sigma)^2)+0.5N(\sigma,(0.1\sigma)^2)$, meaning that we can roughly interpret the noise tolerance of SGBD under Condition \ref{cond:normality} as 
being $\tau^* \approx 1.596\sigma^{-1}$. 
This value is similar to the one of SGLD and exhibit the same dependance with respect to $\sigma$, despite the constant in front being slightly smaller. 

This rough analysis suggests that the amount of Gaussian noise the two algorithms can tolerate while being able to exactly recover the original proposal is similar. 
On the other hand, one can expect that, due to the non-linear use of gradients in SGBD, the two algorithms will exhibit significantly different behavior in the case of noise $\eta_\theta$ with larger variance and/or an heavier-tail distribution, with SGBD providing a more and stable robust behaviour (e.g.\ in the sense of resulting in a proposal closer to the original exact unadjusted scheme). 
We illustrate this phenomenon numerically in this section, leaving a more detailed theoretical analysis of this behaviour to future work.

We consider a standard Normal target, i.e. $\pi(\theta) = \phi(\theta)$ and add to the true gradient symmetric noise with Laplace or Cauchy distribution and scale parameter $\tau_\theta$.
\begin{figure}[h!]
     \centering
     \begin{subfigure}[b]{0.48\textwidth}
         \centering
         \includegraphics[width=\textwidth]{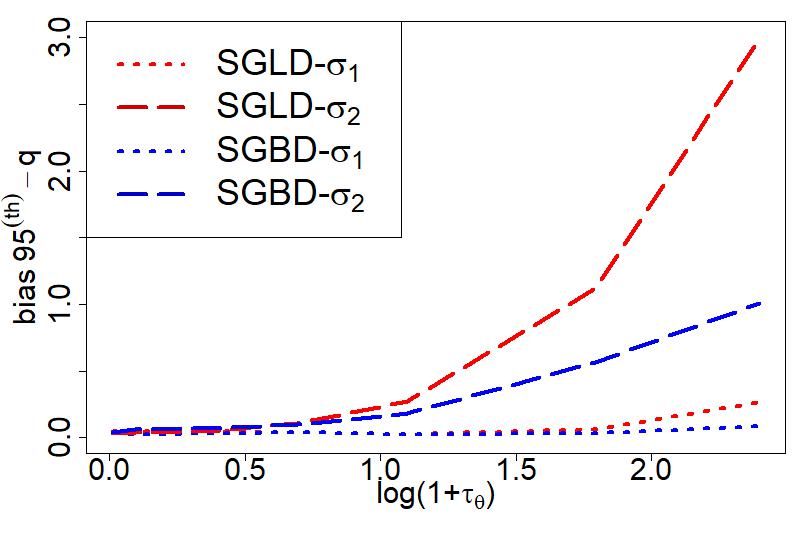}
         \caption{}
         \label{fig:toy_ht_laplace_q_95}
     \end{subfigure}
     \hfill
     \begin{subfigure}[b]{0.48\textwidth}
         \centering
         \includegraphics[width=\textwidth]{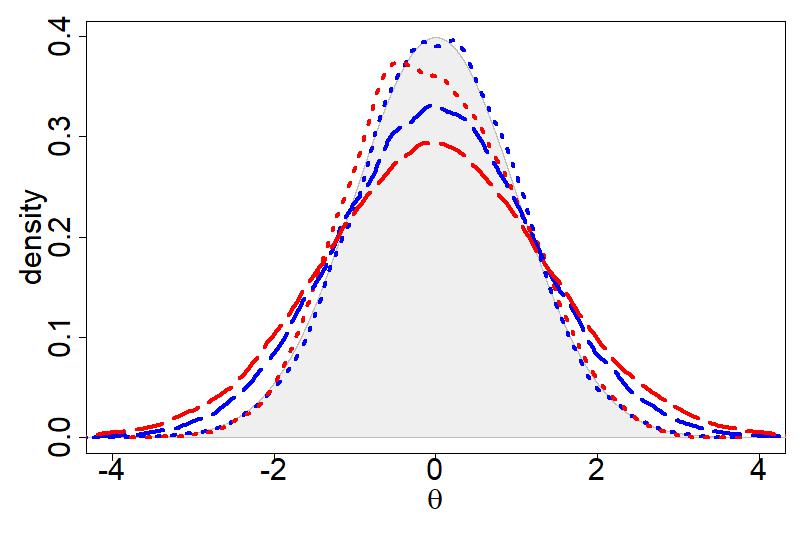}
         \caption{}
         \label{fig:toy_ht_laplace_density}
     \end{subfigure}
       \caption[Toy Example: Normal Target with isotropic Heavy-Tailed Noise]{Toy Example: Normal Target with Laplace Noise in gradient estimates. Scale parameter of the noise distribution (in log-scale) versus bias of the $95^{th}$ quantile (left) and estimate of the stationary distributions when $\tau_\theta = e^{1.5}-1$ (right). Red refers to v-SGLD and blue to v-SGBD. Dotted (dashed resp.) lines are produced with $\sigma_1=0.1$ ($\sigma_2=0.5$ resp.).}
        \label{fig:toy_ht_laplace}
\end{figure}
\begin{figure}[h!]
     \centering
     \begin{subfigure}[b]{0.48\textwidth}
         \centering
         \includegraphics[width=\textwidth]{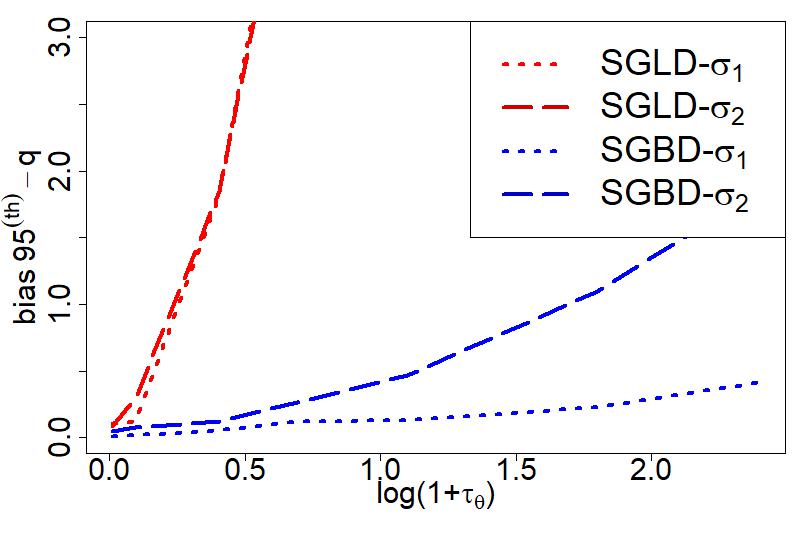}
         \caption{}
         \label{fig:toy_ht_cauchy_q_95}
     \end{subfigure}
     \hfill
     \begin{subfigure}[b]{0.48\textwidth}
         \centering
         \includegraphics[width=\textwidth]{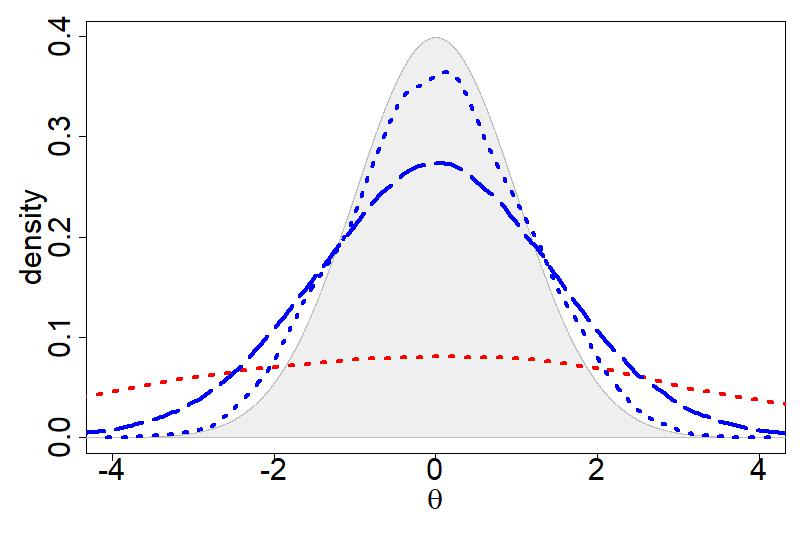}
         \caption{}
         \label{fig:toy_ht_cauchy_density}
     \end{subfigure}
       \caption[Toy Example: Normal Target with isotropic Heavy-Tailed Noise]{Toy Example: Normal Target with Cauchy Noise in gradient estimates. Scale parameter of the noise distribution (in log-scale) versus bias of the $95^{th}$ quantile (left) and estimate of the stationary distributions when $\tau_\theta = e^{1.5}-1$ (right). Red refers to v-SGLD and blue to v-SGBD. Dotted (dashed resp.) lines are produced with $\sigma_1=0.1$ ($\sigma_2=0.5$ resp.). v-SGLD with step-size $\sigma_2$ fails completely to sample from the target and its density estimate is omitted.}
        \label{fig:toy_ht_cauchy}
\end{figure}
Compared to the example presented in Section \ref{subsec:toy_example}, this setting allows to focus on the robustness to heavy-tailed gradient noise rather than skewness in the target distribution. 
We consider two values for the step-size $\sigma$, namely $\sigma_1 = 0.1$ and $\sigma_2 = 0.5$ and run v-SGBD and v-SGLD for $T=2\times 10^5$ discarding the first half as burn-in. 
Figure \ref{fig:toy_ht_laplace_q_95} shows the increase in the bias of the $95^{th}$ quantile of the invariant distribution relative to $\pi$, when $\eta_\theta \sim \mathbf{Laplace}(0, \tau_\theta)$  as $\tau_\theta$ increases, and Figure \ref{fig:toy_ht_laplace_density} reports the density estimates when $\tau_\theta = e^{1.5} - 1$.  With Laplace distributed noise, v-SGLD and v-SGBD with a small step-size behave similarly. However, v-SGBD is more robust than v-SGLD when the step-size is increased and exhibits a smaller bias.
Figure \ref{fig:toy_ht_cauchy} shows the same plots, where the noise is Cauchy distributed, i.e.\  $\eta_\theta \sim \mathbf{Cauchy}(0, \tau_\theta)$. In this scenario, v-SGBD is dramatically more robust than v-SGLD.

\section{ADDITIONAL EXPERIMENTS}
We report the additional simulations details and results. All experiments were run on a laptop with 11th Gen Intel(R) Core(TM) i7-1165G7 2.80 GHz using R version 4.3.1.
\subsection{Empirical Simulation Of The Correction for $p$}
We present additional details for Figure \ref{fig:p_hat_simulation} in the paper. 
Figure \ref{fig:p_hat_simulation_2} reports the same plot for three additional coordinates. 
We take the last sample from the v-SGBD chain used to produce the Figure \ref{fig:log_reg_iter_vs_log_loss} in the paper, we repeatedly subsample a mini-batch, store the gradient for each coordinate, estimate its standard deviation, and compute $\hat{p}$ and $\tilde{p}$. 
\begin{figure}[h!]
     \centering
     \begin{subfigure}[b]{0.48\textwidth}
         \centering
         \includegraphics[width=\textwidth]{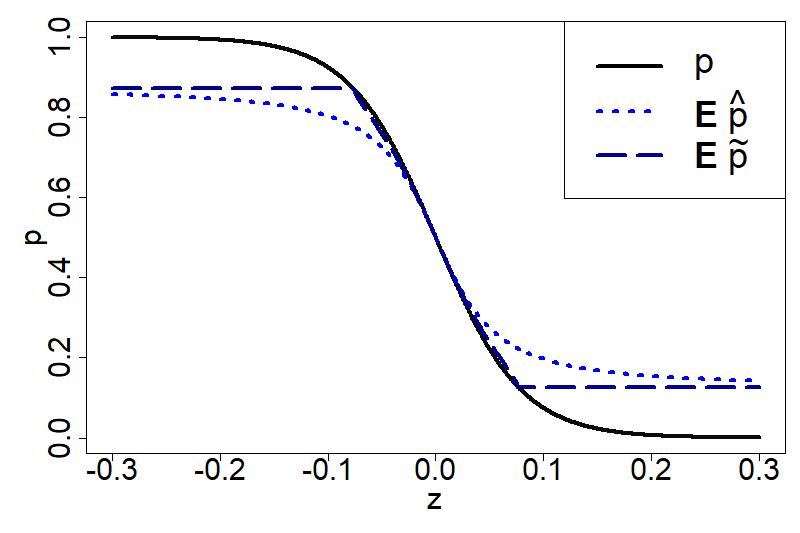}
         \caption{}
         \label{p_hat_plot_a}
     \end{subfigure}
     \hfill
     \begin{subfigure}[b]{0.48\textwidth}
         \centering
         \includegraphics[width=\textwidth]{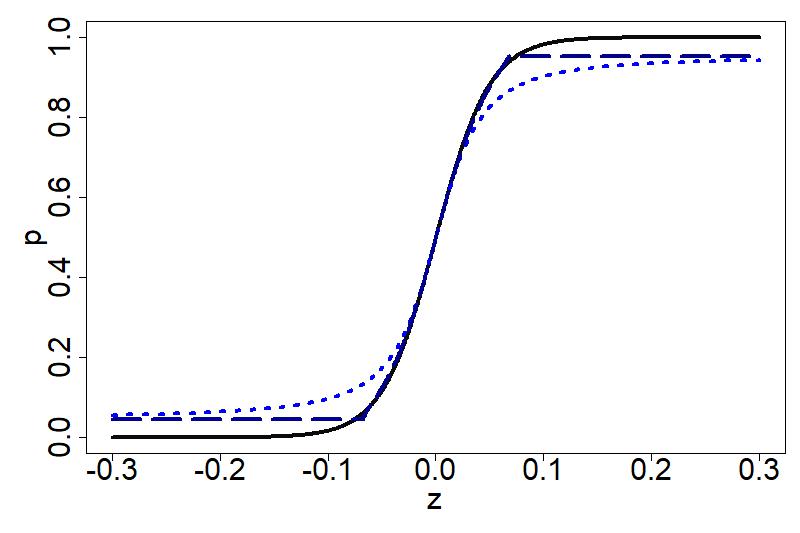}
         \caption{}
         \label{p_hat_plot_b}
     \end{subfigure}
     \vfill
     \begin{subfigure}[b]{0.48\textwidth}
         \centering
         \includegraphics[width=\textwidth]{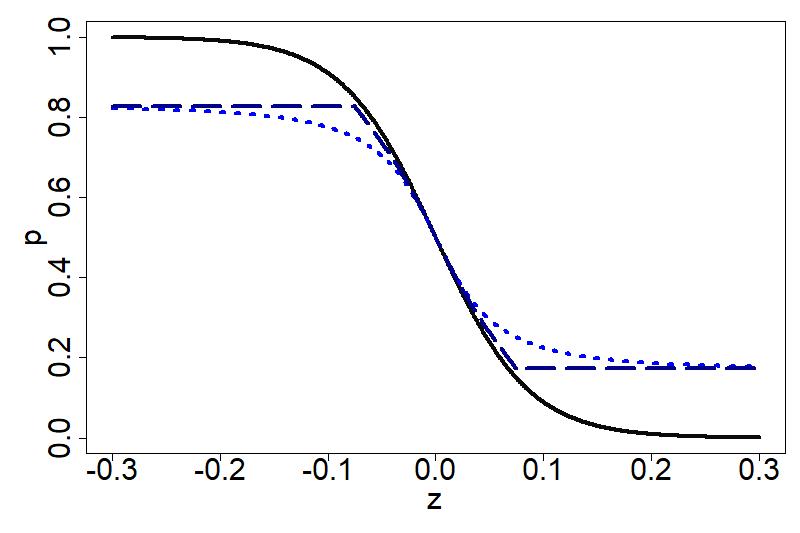}
         \caption{}
         \label{p_hat_plot_c}
     \end{subfigure}
     \hfill
          \begin{subfigure}[b]{0.48\textwidth}
         \centering
         \includegraphics[width=\textwidth]{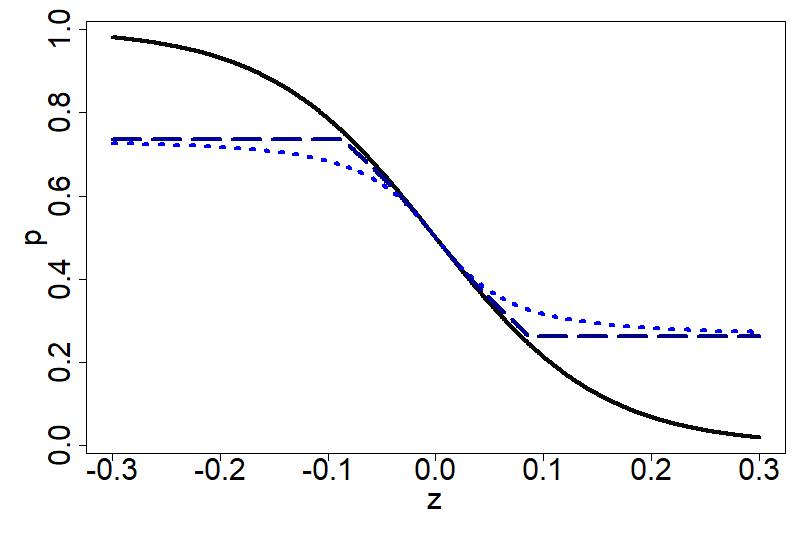}
         \caption{}
         \label{p_hat_plot_d}
     \end{subfigure}
        \vspace{.3in}
        \caption{Shrinkage effect and bias correction. Plot of $p({\partial}_j g(\theta), z)$ (black line; $p$) and Monte Carlo estimates of $\mathbb{E}[p(\hat{\partial}_j g(\theta), z)]$ (dotted blue line; $\textbf{E}\hat{p}$), and $\mathbb{E}[\tilde p(\hat{\partial}_j g(\theta), z)]$ (dashed dark blue line; $\textbf{E}\tilde{p}$) versus the proposed increment $z$; for a logistic regression example with real data.}
        \label{fig:p_hat_simulation_2}
\end{figure}
Figure \ref{fig:p_hat_simulation_2} reports the resulting Monte Carlo average of $\hat{p}$ and $\tilde{p}$ as the value of the proposed increment varies.

The figure was produced with the following values: (a) $\partial_j g(\theta) = -25.15$, $\tau_\theta = 22.72$, (b) $\partial_j g(\theta) = 40.39$, $\tau_\theta = 25.63$, (c) $\partial_j g(\theta) = -23.21$, $\tau_\theta = 22.88$
(d) $\partial_j g(\theta) = -13.03$, $\tau_\theta = 20.27$.

\subsection{Toy Example: Skew-Normal target with isotropic Gaussian Noise}

Figure \ref{toy_plot_1} reports the results obtained using also corrected variants of the algorithms and with a different value for the standard deviation of $\eta_\theta$ for the experiment presented in Section \ref{subsec:toy_example} of the paper. 
\begin{figure}[h!]
     \centering
     \begin{subfigure}[b]{0.48\textwidth}
         \centering
         \includegraphics[width=\textwidth]{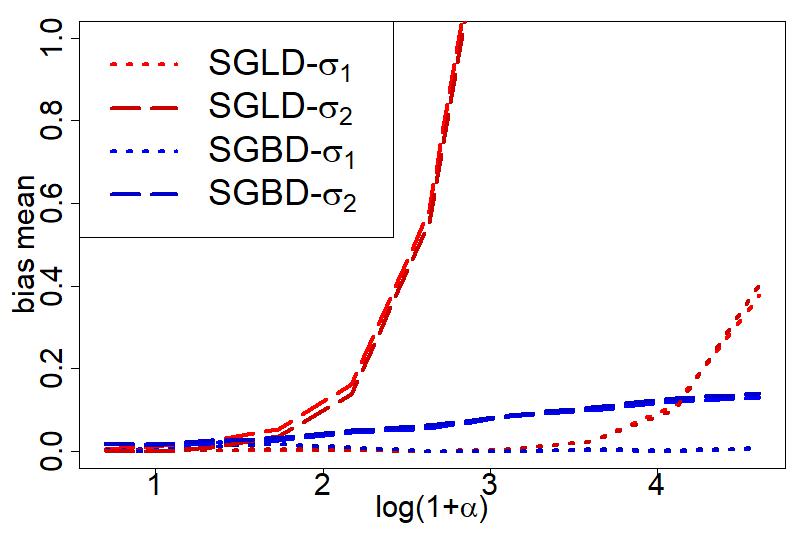}
         \caption{small $\tau_\theta$}
         \label{toy_2_a_plot}
     \end{subfigure}
     \hfill
     \begin{subfigure}[b]{0.48\textwidth}
         \centering
         \includegraphics[width=\textwidth]{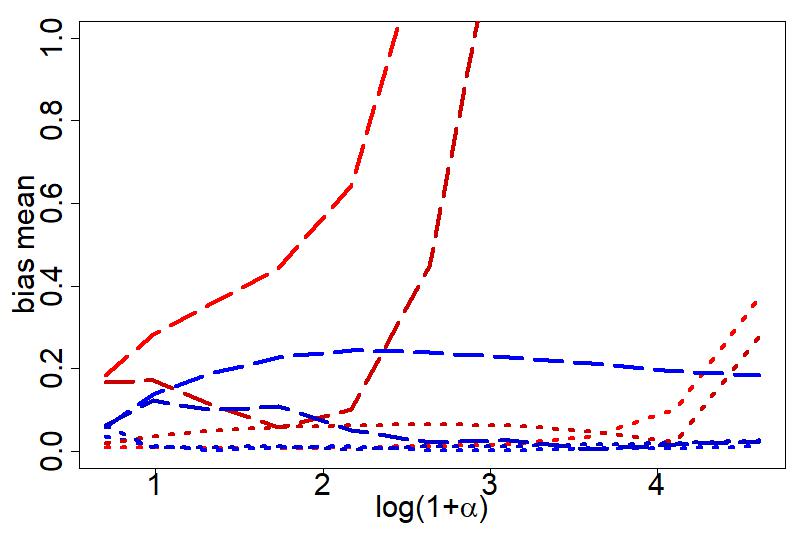}
         \caption{large $\tau_\theta$}
         \label{toy_2_b_plot}
     \end{subfigure}
          \vspace{.3in}
       \caption[Toy Example: Skew-Normal target with isotropic Gaussian Noise]{Toy Example: Skew-Normal target with isotropic Gaussian Noise in gradient estimates. Shape parameter (in log-scale) versus relative bias of mean with noise standard deviation $\tau_\theta = 1 \times sd(\pi_\alpha)$ (left) and $\tau_\theta = 10\times sd(\pi_\alpha)$ (right). Red refers to Langevin-based schemes and blue to Barker-based schemes. Lighter lines refer to vanilla implementations of the algorithm, darker lines to their corrected variants. Dotted (dashed resp.) lines are produced with $\sigma_1=0.1\times sd(\pi_\alpha)$ ($\sigma_2=0.5\times sd(\pi_\alpha)$ resp.).}
        \label{toy_plot_1}
\end{figure}
With a gradient noise standard deviation equal to the one of the target, we observe little difference between vanilla and corrected variants of the algorithms (Figure \ref{toy_2_a_plot}). 
When the stardard deviation if large, i.e.\ $\tau_\theta = 10\times sd(\pi_\alpha)$, corrected variants achieve a lower bias then their vanilla counterparts with a large step-size (Figure \ref{toy_2_b_plot}). Under all configurations, SGBD displays increased robustness to skewness and to tuning with respect to SGLD.

Next, we perform a sensitivity analysis to the step-size choice. Figure \ref{fig:toy_skew_sensitivity_analysis} reports the density estimates of the samples obtained via v-SGBD with the step-size respectively equal to $0.05\times sd(\pi_\alpha)$, $0.1\times sd(\pi_\alpha)$, $0.5\times sd(\pi_\alpha)$, and $0.75\times sd(\pi_\alpha)$, for $\alpha=20$, where $ sd(\pi_\alpha)$ denotes the standard deviation of the target distribution. v-SGBD appears to be very robust to the step-size choice, and only slightly inflates the scale of the target distribution under the configuration with the largest step-size. 
\begin{figure}[h!]
     \centering
     \begin{subfigure}[b]{0.48\textwidth}
         \centering
         \includegraphics[width=\textwidth]{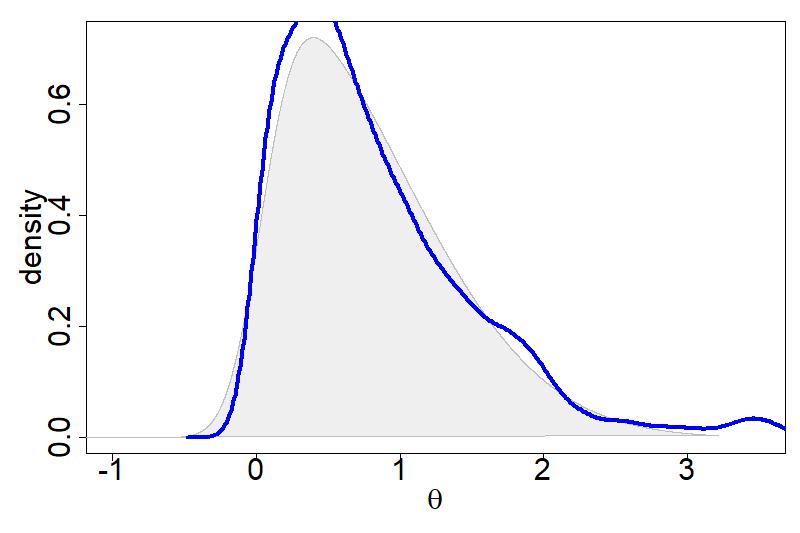}
         \label{}
     \end{subfigure}
     \hfill
     \begin{subfigure}[b]{0.48\textwidth}
         \centering
         \includegraphics[width=\textwidth]{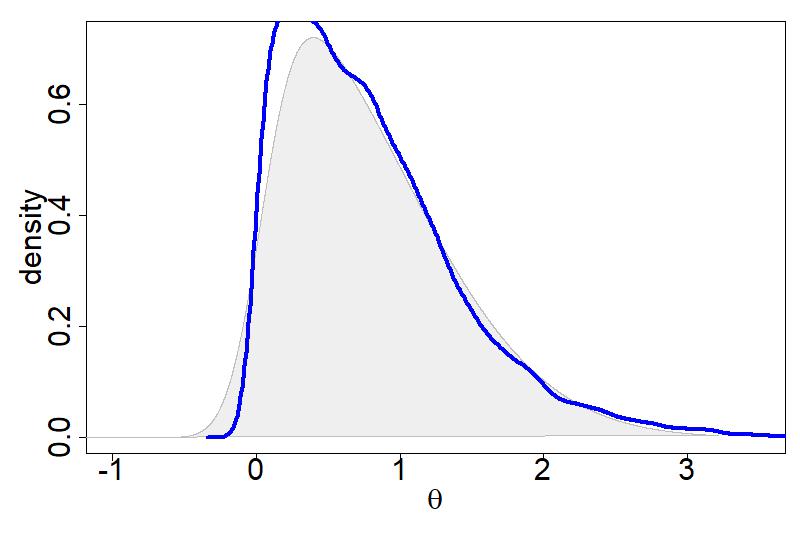}
         \label{}
     \end{subfigure}
     \vfill
     \begin{subfigure}[b]{0.48\textwidth}
         \centering
         \includegraphics[width=\textwidth]{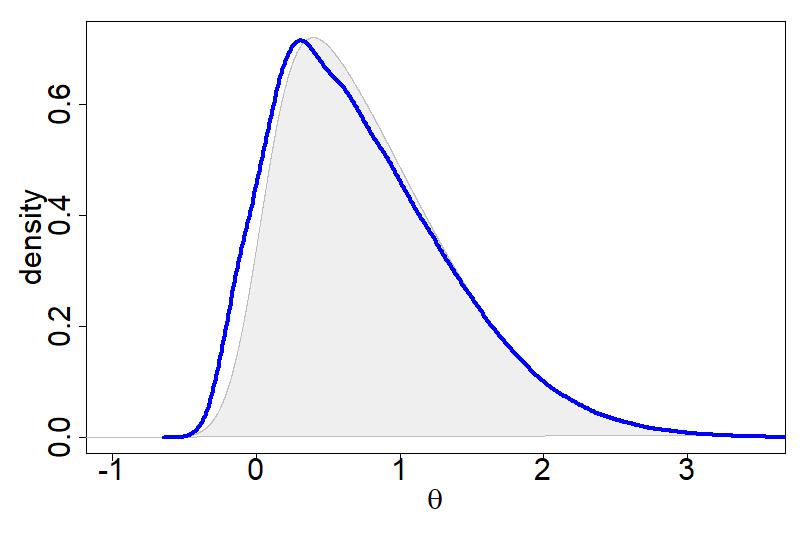}
         \label{}
     \end{subfigure}
     \hfill
    \begin{subfigure}[b]{0.48\textwidth}
         \centering
         \includegraphics[width=\textwidth]{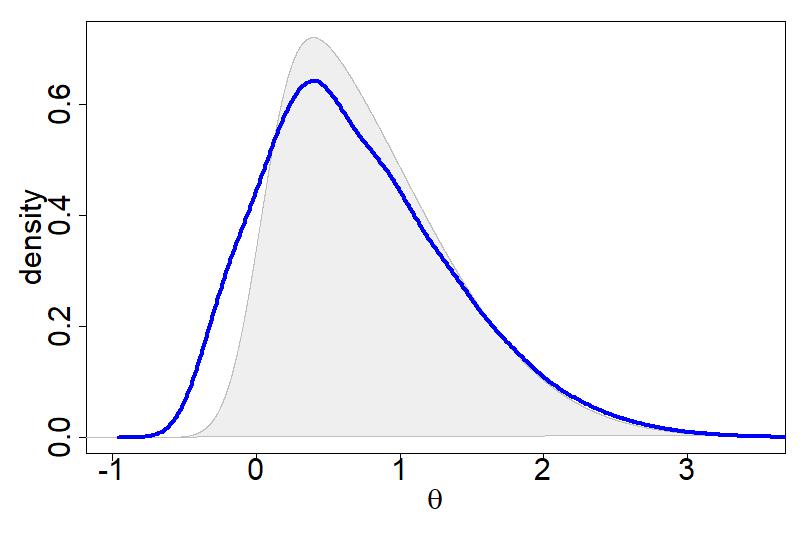}
         \label{}
     \end{subfigure}
          \vspace{.3in}
        \caption[Toy Example: Skew-Normal target with isotropic Gaussian Noise]{Toy Example: Skew-Normal target with isotropic Gaussian Noise in gradient estimates. Density estimates of the samples obtained via v-SGBD with the step-size respectively equal to $0.05\times sd(\pi_\alpha)$, $0.1\times sd(\pi_\alpha)$, $0.5\times sd(\pi_\alpha)$, and $0.75\times sd(\pi_\alpha)$, for $\alpha=20$.}
        \label{fig:toy_skew_sensitivity_analysis}
\end{figure}

\subsection{Binary Regression With Scale Heterogeneity}

Figure \ref{sepsis_log_reg_hist_1} shows the density estimates for all coordinates of the posterior samples obtained in the experiment in Section \ref{subsec:log_reg_sepsis}, and Figure \ref{traceplots:log_reg_sepsis_appendix} shows the corresponding trace plots. 
\begin{figure}[h!]
     \centering
     \begin{subfigure}[b]{0.48\textwidth}
         \centering
         \includegraphics[width=\textwidth]{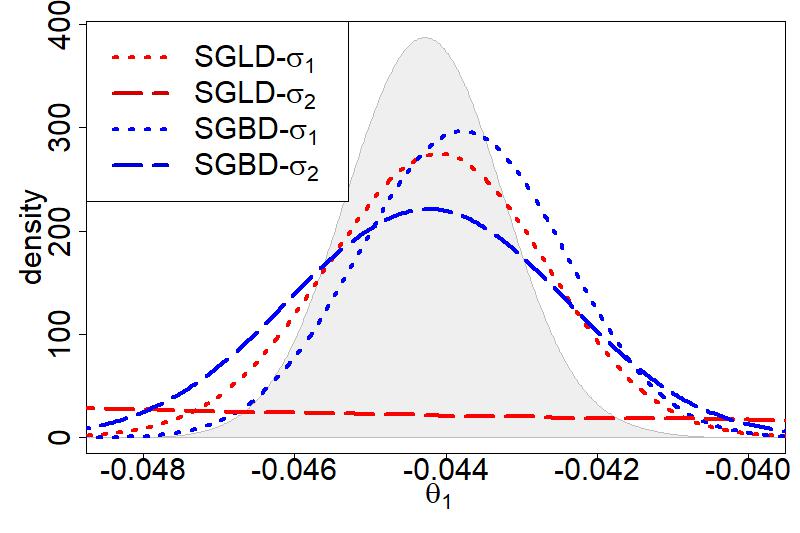}
         \label{}
     \end{subfigure}
     \hfill
     \begin{subfigure}[b]{0.48\textwidth}
         \centering
         \includegraphics[width=\textwidth]{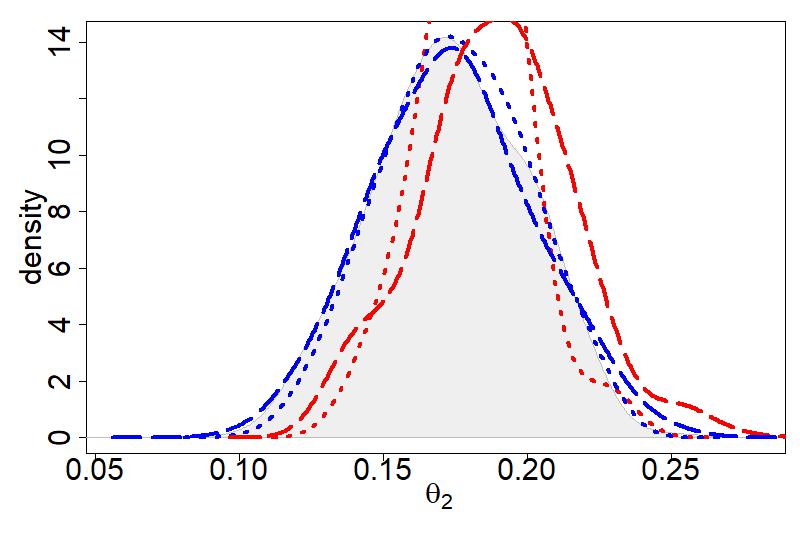}
         \label{}
     \end{subfigure}
     \vfill
     \begin{subfigure}[b]{0.48\textwidth}
         \centering
         \includegraphics[width=\textwidth]{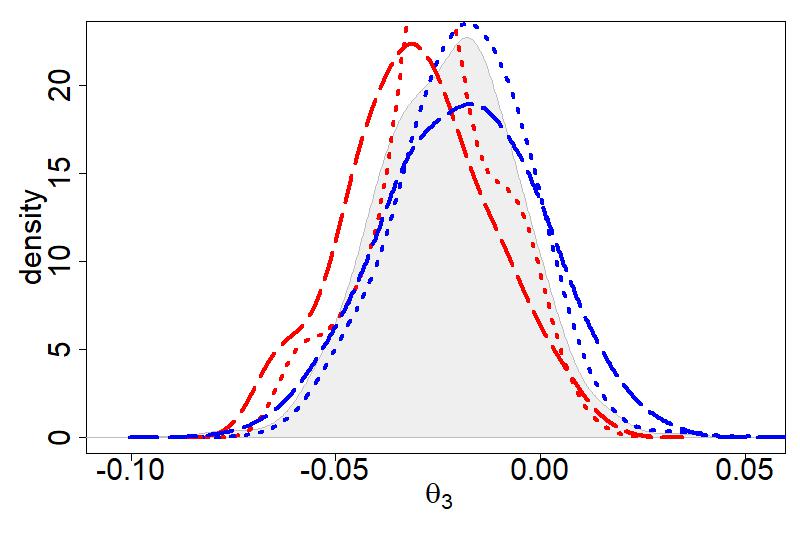}
         \label{}
     \end{subfigure}
     \hfill
    \begin{subfigure}[b]{0.48\textwidth}
         \centering
         \includegraphics[width=\textwidth]{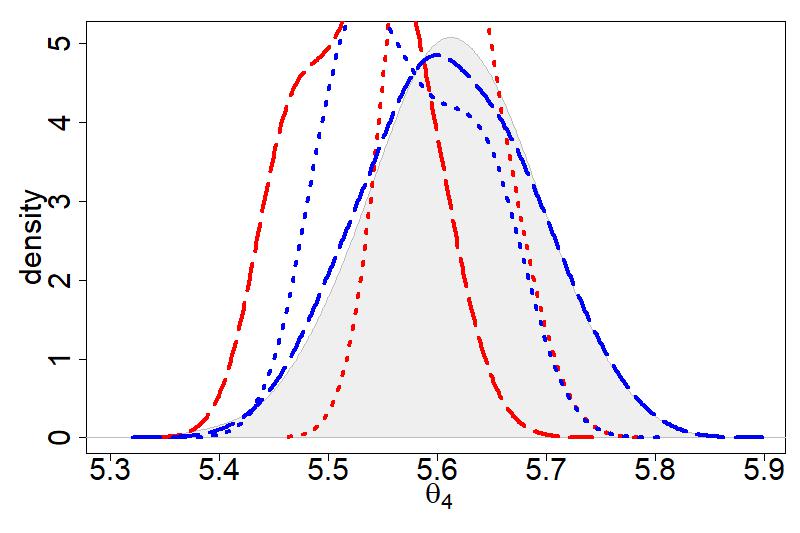}
         \label{}
     \end{subfigure}
          \vspace{.3in}
        \caption[Logistic Regression Univariate distributions with heterogeneity of scale]
        {Logistic Regression with Scale Heterogeneity on the Sepsis dataset. Blue (red resp.) lines represent the density estimates of v-SGBD (v-SGLD) samples. Dotted (dashed resp.) lines are produced using a smaller (larger resp.) step size. Grey areas represent the estimate of the marginal posterior densities obtained with STAN.}
        \label{sepsis_log_reg_hist_1}
\end{figure}
\begin{figure}[h!]
     \centering
     \begin{subfigure}[b]{0.48\textwidth}
         \centering
         \includegraphics[width=\textwidth]{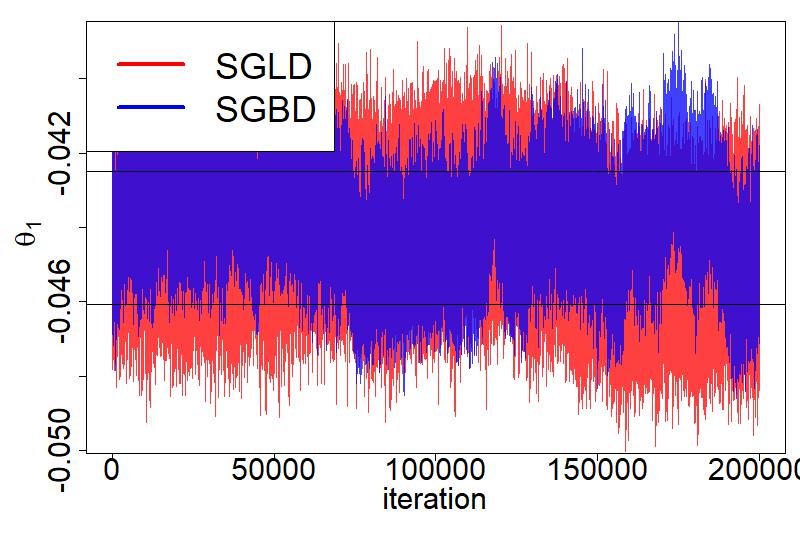}
         \label{traceplot:log_reg_sepsis_1}
     \end{subfigure}
     \hfill
     \begin{subfigure}[b]{0.48\textwidth}
         \centering
         \includegraphics[width=\textwidth]{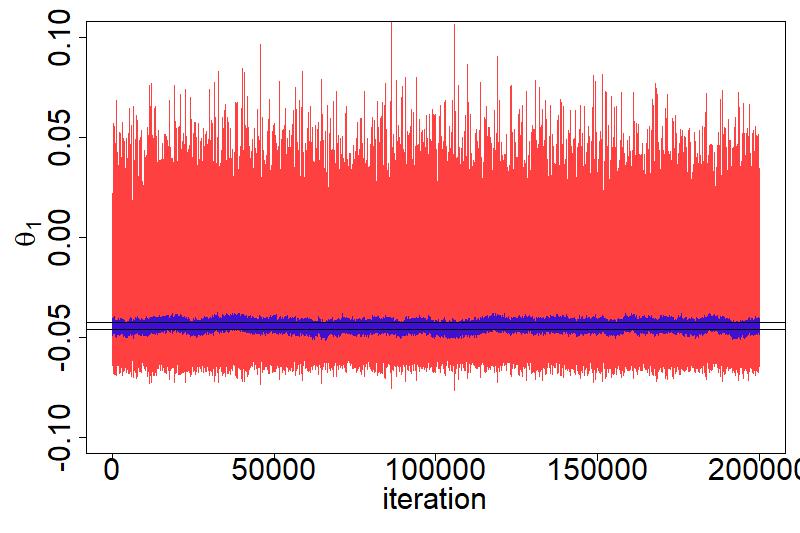}
         \label{traceplot:log_reg_sepsis_2}
     \end{subfigure}
     \vfill
     \begin{subfigure}[b]{0.48\textwidth}
         \centering
         \includegraphics[width=\textwidth]{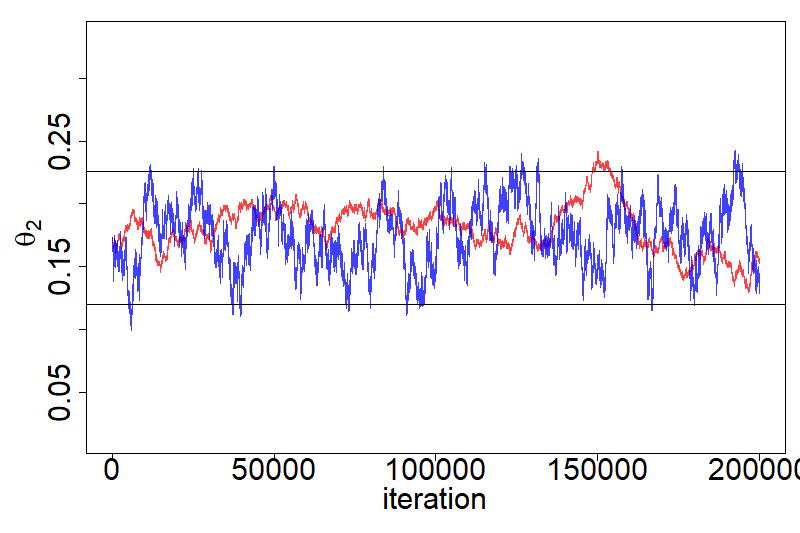}
         \label{traceplot:log_reg_sepsis_3}
     \end{subfigure}
     \hfill
     \begin{subfigure}[b]{0.48\textwidth}
         \centering
         \includegraphics[width=\textwidth]{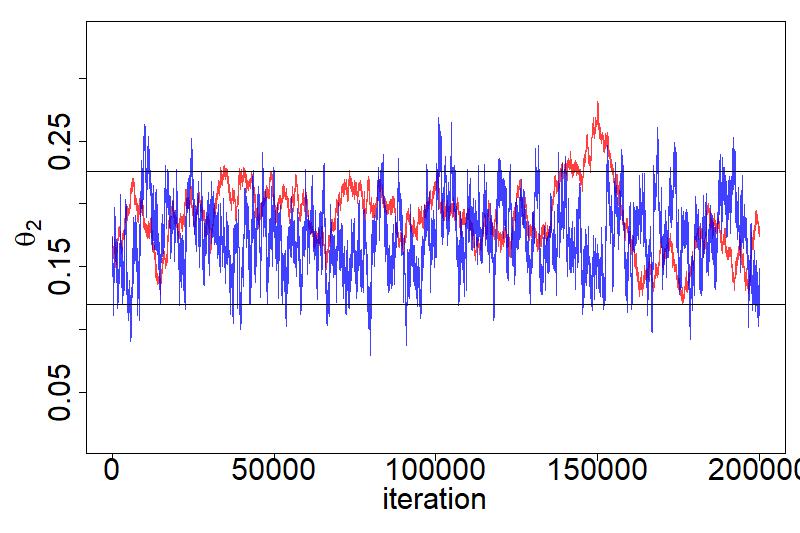}
         \label{traceplot:log_reg_sepsis_4}
         \end{subfigure}
    \begin{subfigure}[b] {0.48\textwidth}
         \centering
         \includegraphics[width=\textwidth]{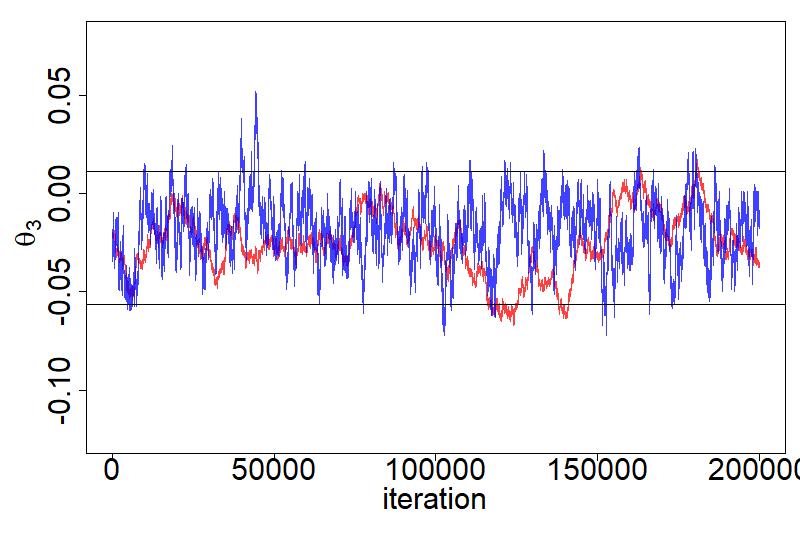}
         \label{traceplot:log_reg_sepsis_1}
     \end{subfigure}
     \hfill
     \begin{subfigure}[b]{0.48\textwidth}
         \centering
         \includegraphics[width=\textwidth]{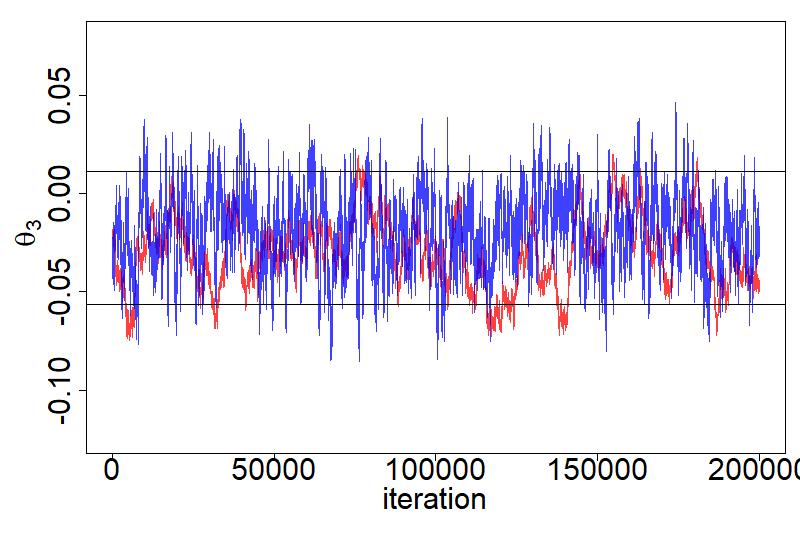}
         \label{traceplot:log_reg_sepsis_2}
     \end{subfigure}
     \end{figure}%
\begin{figure}[t]\ContinuedFloat
     \begin{subfigure}[b]{0.48\textwidth}
         \centering
         \includegraphics[width=\textwidth]{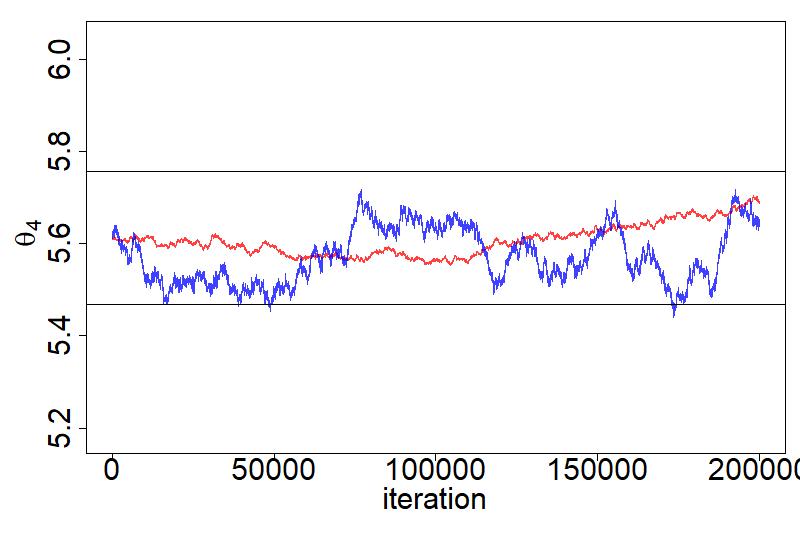}
         \label{traceplot:log_reg_sepsis_3}
     \end{subfigure}
     \hfill
     \begin{subfigure}[b]{0.48\textwidth}
         \centering
         \includegraphics[width=\textwidth]{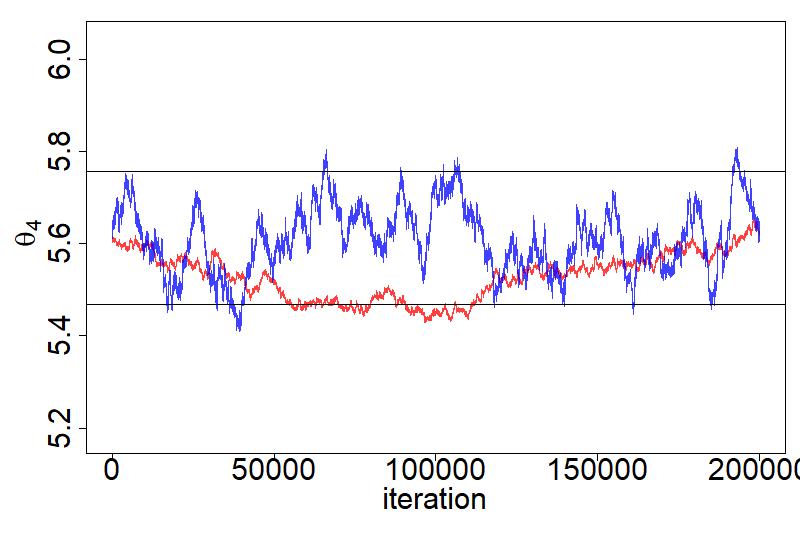}
         \label{traceplot:log_reg_sepsis_4}
     \end{subfigure}
          \vspace{.3in}
        \caption[Logistic Regression with Scale Heterogeneity]{Logistic Regression with Scale Heterogeneity on the Sepsis dataset. Traceplots of all the coordinates with two step-size configurations: small $\sigma$ (left), large $\sigma$ (right). Blue refers to v-SGBD and red to v-SGLD. Black horizontal lines represent the interval around the posterior mean with a two standard deviations width.}
        \label{traceplots:log_reg_sepsis_appendix}
    \end{figure}

Figure \ref{fig:sepsis_log_reg_mamba} reports the density estimates for all coordinates of the posterior samples from a second experiment where we tuned the step-sizes via MAMBA \citep{mamba}. MAMBA uses a multi-armed bandits algorithm to minimize the Finite Set Stein Discrepancy \citep{fssd} between true posterior and its Monte
Carlo approximation. 
MAMBA selects a step-size equal to $3.24 \times 10^{-4}$ and $1.36 \times 10^{-3}$ for v-SGLD and v-SGBD respectively. In general, the chosen step-sizes match the largest scale of the coordinates but are too large for the first one. We note that SGBD outperforms SGLD in particular for the first coordinate where SGLD remarkably inflates its variance. \begin{figure}[h!]
     \centering
     \begin{subfigure}[b]{0.48\textwidth}
         \centering
         \includegraphics[width=\textwidth]{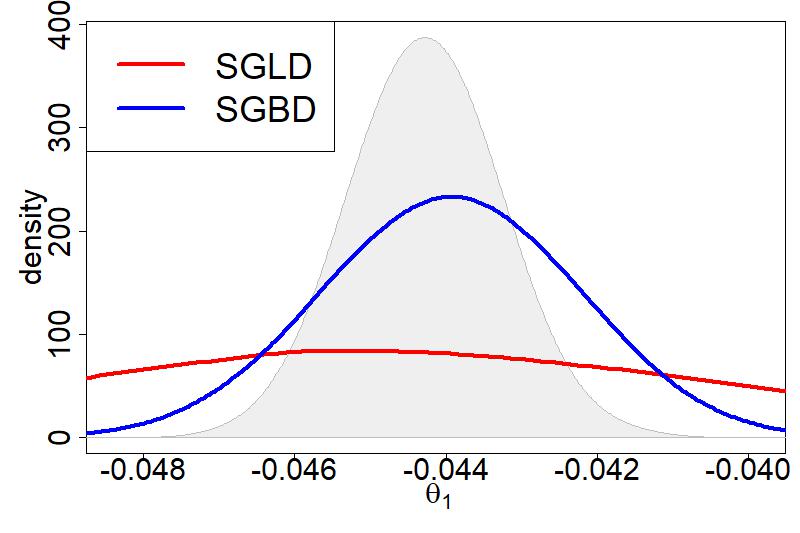}
         \label{}
     \end{subfigure}
     \hfill
     \begin{subfigure}[b]{0.48\textwidth}
         \centering
         \includegraphics[width=\textwidth]{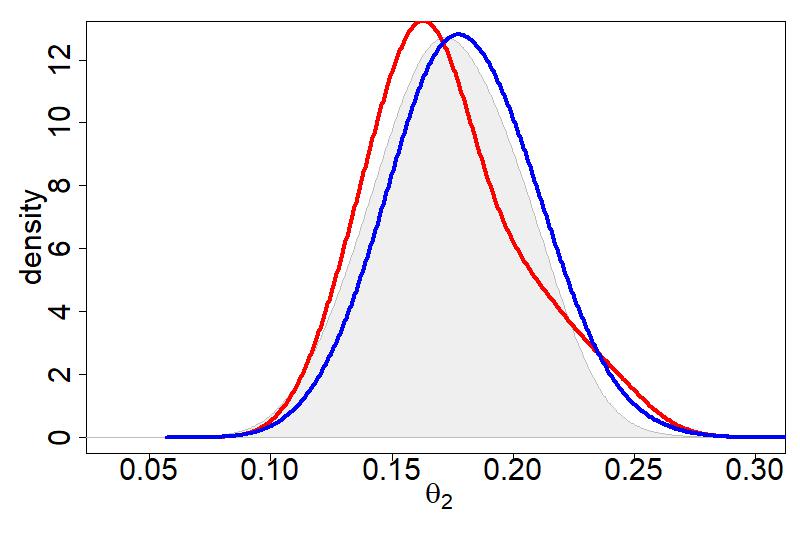}
         \label{}
     \end{subfigure}
     \vfill
     \begin{subfigure}[b]{0.48\textwidth}
         \centering
         \includegraphics[width=\textwidth]{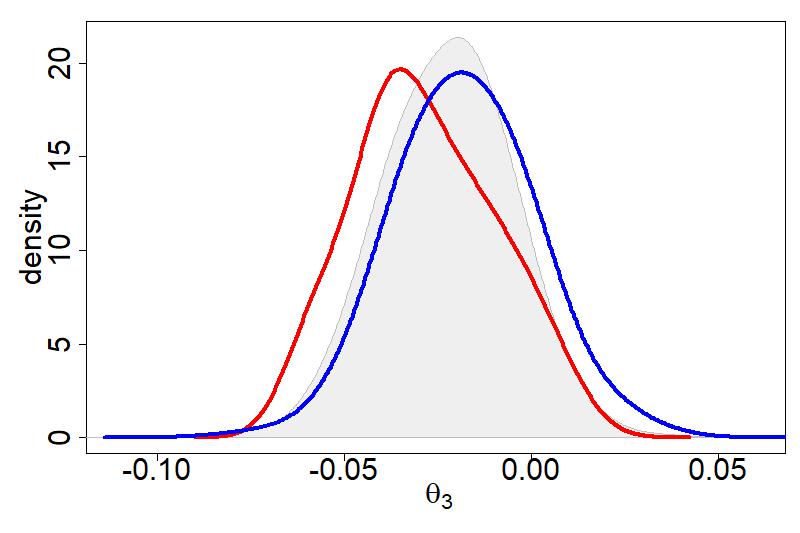}
         \label{}
     \end{subfigure}
     \hfill
    \begin{subfigure}[b]{0.48\textwidth}
         \centering
         \includegraphics[width=\textwidth]{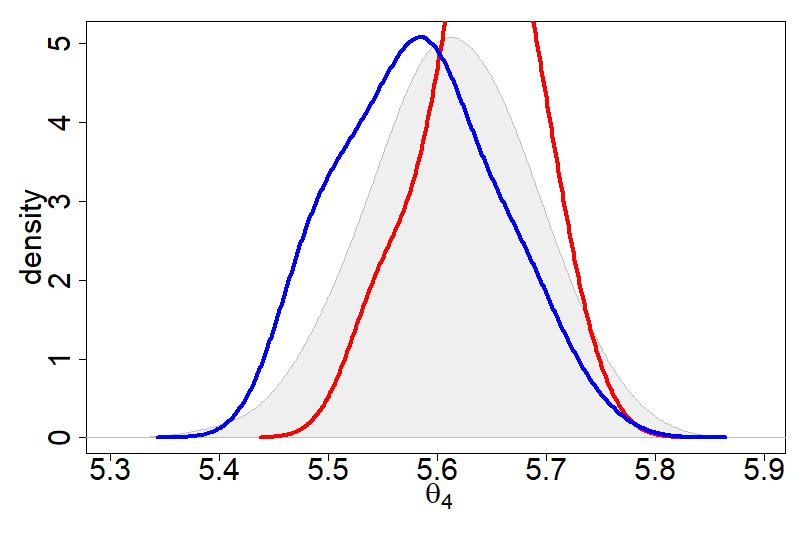}
         \label{}
     \end{subfigure}
          \vspace{.3in}
        \caption[Logistic Regression Univariate distributions with heterogeneity of scale]{Logistic Regression with Scale Heterogeneity on the Sepsis dataset. Blue (red resp.) lines represent the density estimates of v-SGBD (v-SGLD) samples tuning the step-size via MAMBA. Grey areas represent the estimate of the marginal posterior densities obtained with STAN. }
        \label{fig:sepsis_log_reg_mamba}
\end{figure}

Figure \ref{sepsis_log_reg_plot_2} reports the density estimates for all coordinates of the posterior samples from a third experiment where we selected different step-sizes for each coordinated. In particular, we are interested in the case where a diagonal preconditioner is applied and we set $\sigma_{1,j}=0.1\times sd(\pi_j)$ and $\sigma_{2,j}=0.2\times sd(\pi_j)$, for $j=1, \dots, 4$, where $sd(\pi_j)$ denotes the standard deviation of the posterior distribution for $\theta_j$. The aim of this experiment is to study the performance of the algorithm with a correct tuning of the step-size across the coordinates. However, we note that this is in general not easily to do in practice as it requires to have access to posterior quantities which are in general unknown a priori and usually estimated with MCMC. 
\begin{figure}[h!]
     \centering
     \begin{subfigure}[b]{0.48\textwidth}
         \centering
         \includegraphics[width=\textwidth]{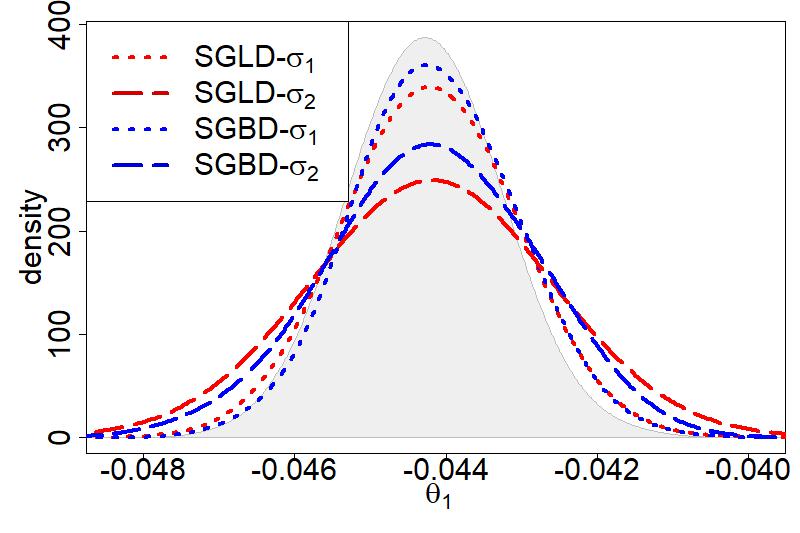}
         \label{}
     \end{subfigure}
     \hfill
     \begin{subfigure}[b]{0.48\textwidth}
         \centering
         \includegraphics[width=\textwidth]{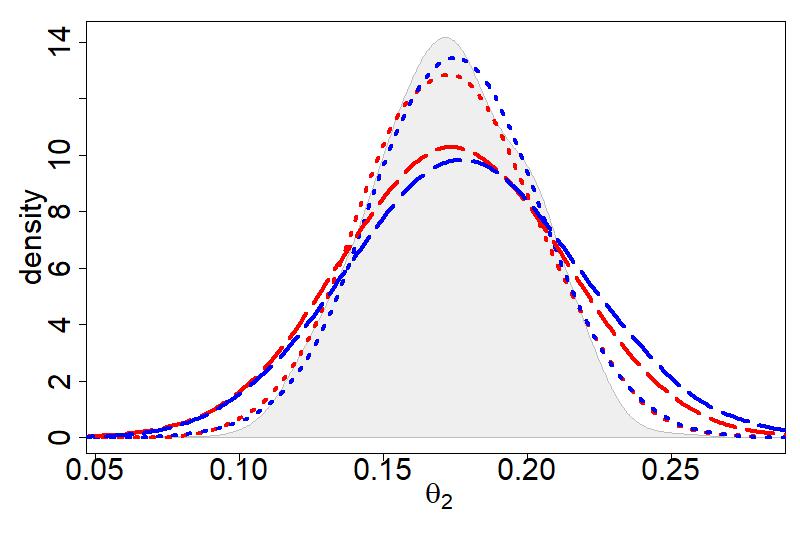}
         \label{}
     \end{subfigure}
     \vfill
     \begin{subfigure}[b]{0.48\textwidth}
         \centering
         \includegraphics[width=\textwidth]{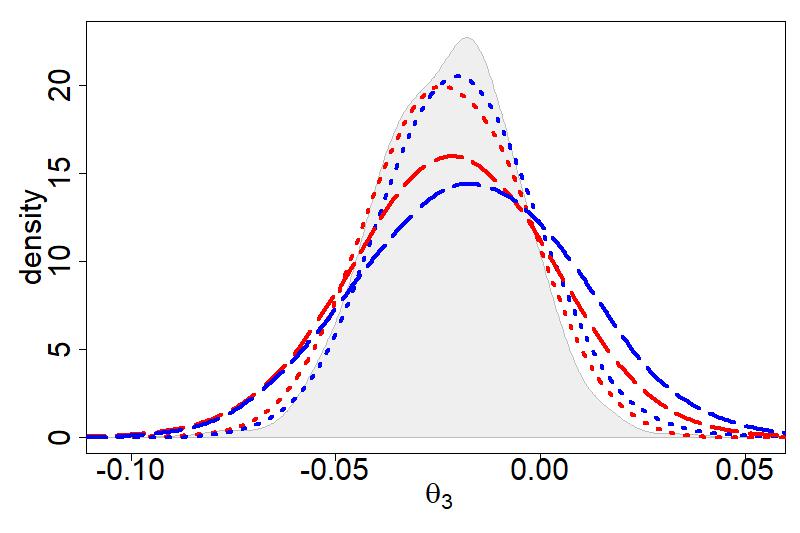}
         \label{}
     \end{subfigure}
     \hfill
    \begin{subfigure}[b]{0.48\textwidth}
         \centering
         \includegraphics[width=\textwidth]{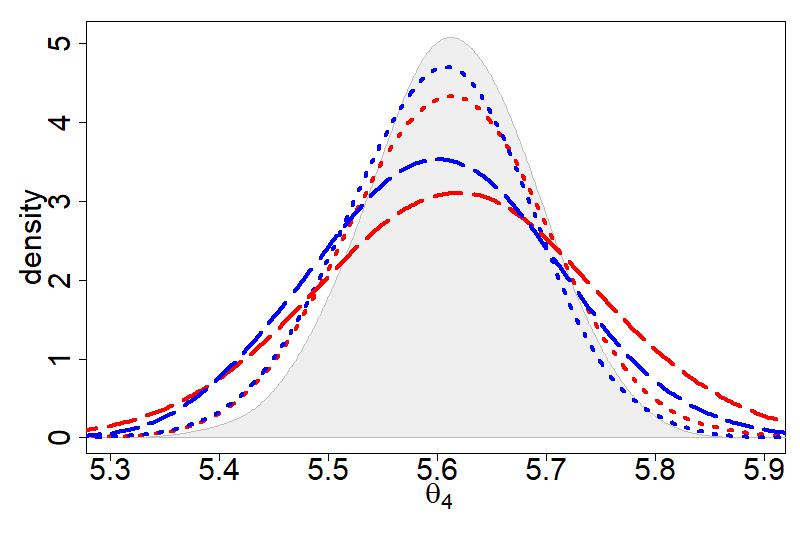}
         \label{}
     \end{subfigure}
          \vspace{.3in}
        \caption[Logistic Regression Univariate distributions with heterogeneity of scale]{Logistic Regression with Scale Heterogeneity on the Sepsis dataset. Blue (red resp.) lines represent the density estimates of v-SGBD (v-SGLD) samples using different step-sizes for each coordinate. Dotted (dashed resp.) lines are produced with $\sigma_{1,j}=0.1\times sd(\pi_j)$ ($\sigma_{2,j}=0.2\times sd(\pi_j)$ resp.) for $j=1, \dots, 4$. Grey areas represent the estimate of the marginal posterior densities obtained with STAN. }
        \label{sepsis_log_reg_plot_2}
\end{figure}
In this scenario, clearly both algorithms perform better than in the previous one and sample accurately with small step-sizes (dotted lines in Figure \ref{sepsis_log_reg_plot_2}). With a larger step-size (dashed lines in Figure \ref{sepsis_log_reg_plot_2}), the samplers moderately inflate the variance of the marginals distribution with SGBD performing slightly better.

\clearpage
\subsection{Binary Regression With High-Dimensional Predictors}
This section studies the performance of SGBD on ill-conditioned high dimensional logistic regression task using model \eqref{eq:log_reg_model}.
We apply the model to the \href{https://archive.ics.uci.edu/ml/datasets/arrhythmia}{Arrhythmia dataset} from the UCI repository. The dataset contains 452 instances and 279 covariates, from which we retain the first 100. A random $80-20\%$ train-test split is applied to the dataset, and we run the samplers for $T=10^5$, discarding the first half as burn-in, iteration using a mini-batch of $n=34$.

We are interested in how hyperparameter tuning affects the sampling accuracy of the algorithm. 
In particular, we study the trade-off between mixing and sampling accuracy, since increasing the step size of the algorithms produces better mixing but less accurate chains, as no MH step is used. 
Figure \ref{fig:arr_log_reg_mixing_accuracy} shows how sampling accuracy decreases as mixing increases when the step sizes vary. 
Mixing is measured with the median effective sample size (ESS) across the parameters and sampling accuracy with the mean standardized $1^{st}$ and $2^{nd}$ order bias,
which are defined as follows:
\begin{equation}
\begin{aligned}
      &\textit{Bias}\left(\mathbb E [\theta_j \mid \{y_i\}_{i=1}^n]\right) = \frac{\left|\Bar{\theta}_j - \mathbb E [\theta_j \mid \{y_i\}_{i=1}^n]\right|}{\left(\mathbb {V} [\theta_j \mid \{y_i\}_{i=1}^n]\right)^{1/2}}\\
      &\textit{Bias}\left(\mathbb V [\theta_j \mid \{y_i \}_{i=1}^n]\right) = \frac{\left|\Bar{\tau}_{\theta_i}^2 - \mathbb V [\theta_j \mid \{y_i \}_{i=1}^n]\right|}{\left(\mathbb {V} [\theta_j \mid \{y_i \}_{i=1}^n]\right)^{1/2}}
\end{aligned} \quad j=1, \dots d,
\end{equation}
where $\bar{\theta}_j = \frac{\sum_{t=1}^T \theta_j^{(t)}}{T}$ and $\Bar{\tau}^2_{\theta_j} = \frac{\sum_{t=1}^T\left(\theta_{j}^{(t)} - \bar{\theta}_j\right)^2}{T-1}$.
\begin{figure}[h!]
     \centering
     \begin{subfigure}[b]{0.48\textwidth}
         \centering
         \includegraphics[width=\textwidth]{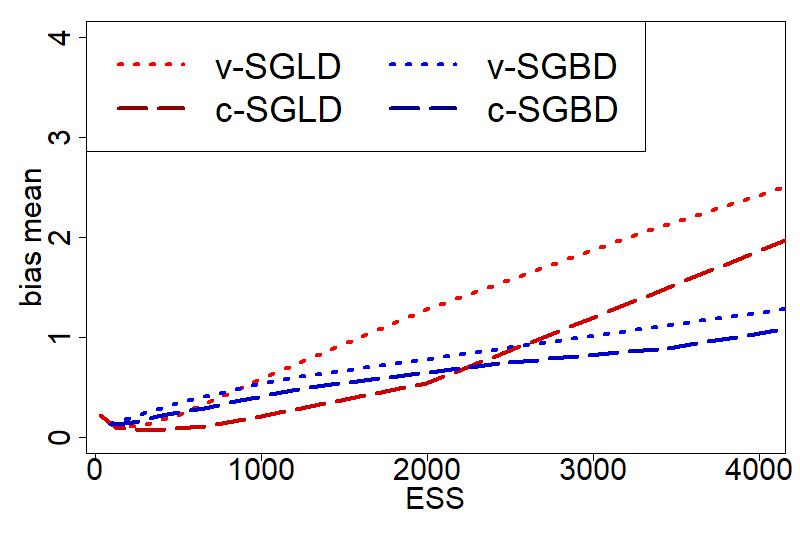}
         \caption{ESS vs bias of mean}
         \label{log_reg_ess_vs_bias_mean}
     \end{subfigure}
     \hfill
     \begin{subfigure}[b]{0.48\textwidth}
         \centering
         \includegraphics[width=\textwidth]{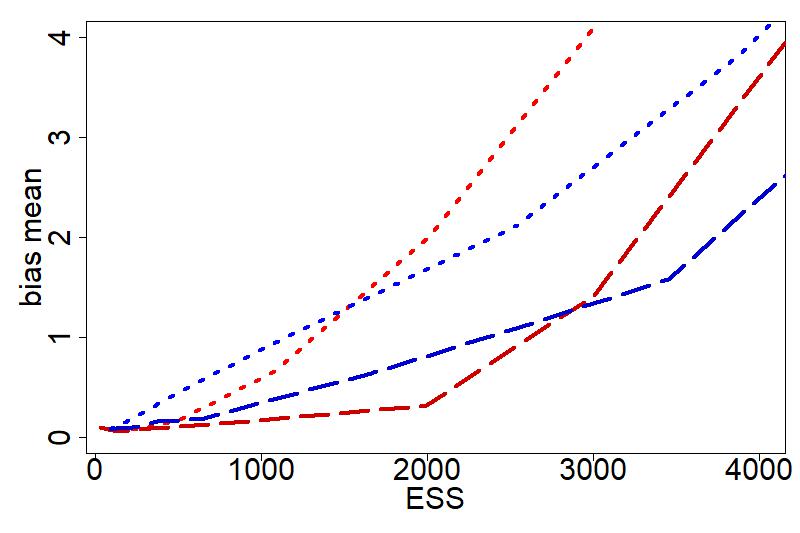}
         \caption{ESS vs bias of variance}
         \label{log_reg_ess_vs_bias_var}
     \end{subfigure}
          \vspace{.3in}
        \caption[Logistic Regression Mixing-Accuracy trade off]{Logistic Regression Mixing-Accuracy trade off on the Arrhythmia dataset. Median ESS versus mean bias of mean (left) and mean bias of variance (right). Red refers to SGLD and blue to the SGBD. Lighter dotted lines refer to vanilla implementations of the algorithm, darker dashed lines to their corrected variants. 
        }
        \label{fig:arr_log_reg_mixing_accuracy}
\end{figure}
The results in Figure \ref{fig:arr_log_reg_mixing_accuracy} suggests that SGLD is more accurate for small step-sizes and low mixing, while SGBD appears to be more robust to hyperparameter tuning. In particular, it enjoys a more favourable mixing-accuracy trade off when the step size is chosen increasingly large. 

Figure \ref{fig:log_reg_iter_vs_log_loss} reports the predictive performance in terms of log-loss on the held-out test set, when hyperparameters are chosen such that the median ESS of the samples is roughly equal to 1000 and are set to $\sigma=0.25, 0.22, 0.07$ for v-SGBD, c-SGBD and e-SGBD and $\sigma=0.14, 0.11, 0.09$ for v-SGLD, c-SGLD and e-SGLD. The log-loss at iteration $t$ is defined as
\begin{equation}
    l(t) =- \frac{1}{|T|} \sum_{i \in \mathcal T} y_i \log\left(\hat{p}(\mathbf{x}_i, \theta)^{(t)}\right) + (1- y_i) \log \left(1 - \hat{p}(\mathbf{x}_i, \theta)^{(t)}\right)
\end{equation}
where $T$ is the test set, $|T|$ denotes its size, and $\hat{p}(\mathbf{x}_i, \theta)^{(t)} = \frac{1}{t}\sum_{k=1}^t\left(1+e^{-x_i^{\top}\theta^{(k)}}\right)^{-1}$ is the ergodic average of the estimate of probability of $Y_i = 1$ given the predictors and samples for the parameter $\theta$. 
SGBD outperforms SGLD in terms of predictive accuracy.

Figure \ref{fig:log_reg_ess_vs_log_loss} reports how the log-loss using all the samples from each chain varies for each configuration of hyperparameter.
In general, SGBD achieves a better predictive accuracy than SGLD with different step-sizes configurations.

\begin{figure}

\begin{subfigure}[b]{0.48 \textwidth}
         \centering
         \includegraphics[width=\textwidth]{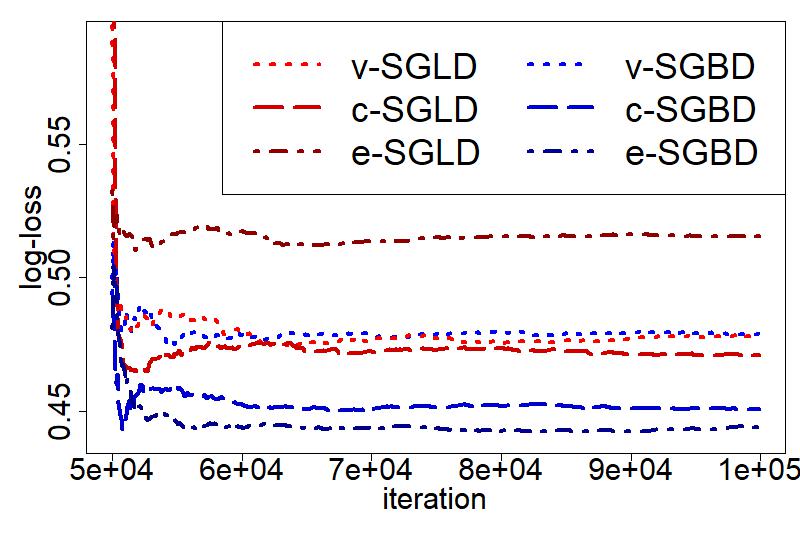}
         \caption{}
\label{fig:log_reg_iter_vs_log_loss}
\end{subfigure}
     \hfill
\begin{subfigure}[b]{0.48 \textwidth}
         \centering
         {\includegraphics[width=\textwidth]{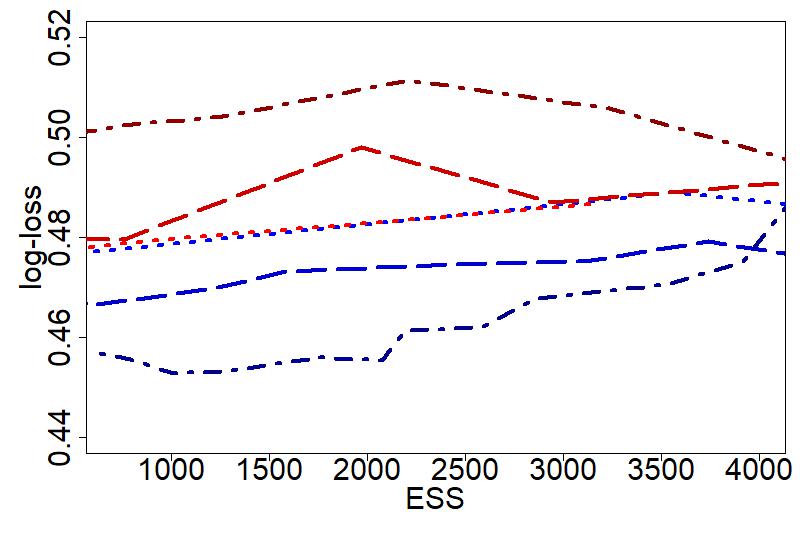}}
         \caption{}
    \label{fig:log_reg_ess_vs_log_loss}
\end{subfigure}
     \vspace{.3in}
\caption[Logistic Regression Predictive Accuracy]{Logistic Regression Predictive Accuracy on the Arrhythmia dataset. Iteration versus log-loss of MCMC estimates for a given configuration of step-size (left) and Median ESS versus log-loss using all samples after the burn-in (right). Red refers to SGLD and blue to the SGBD. For both algorithms, the vanilla (lighter dotted lines), corrected (medium scale dashed lines) and extreme (darker dotted-dashed lines) versions are displayed. 
} 
\end{figure}

\subsection{Additional details for the numerics in Sections \ref{sec.pmf} and \ref{subsec:ica}}

This section reports additional details about the experiments in Section \ref{sec.pmf} and \ref{subsec:ica} of the paper.

In Section \ref{sec.pmf}, we run the algorithms with the following value for the hyperparameter $\sigma=0.022$ and $\sigma=0.005$ for v-SGLD and e-SGLD and $\sigma=0.011$ and $\sigma=0.0105$ for v-SGBD and e-SGBD. 
The rMSE reported in Figure \ref{fig:bpmf} is computed after clipping the ratings predicted values at $1$ and $5$. 
In particular, the sample rMSE (s-rMSE) at iteration $t$ is computed as 
\begin{equation}
    \label{eq:sample_rMSE_bpmf}
    \text{s-rMSE}(t)= \sqrt{\frac{1}{|T|} \sum_{i,j : (i, j) \in T} \left(R_{ij} - \hat R_{ij}^{(t)}\right)}, \quad \hat R_{ij}^{(t)} = \begin{cases}
    1 \quad \text{if } \tilde R_{ij}^{(t)} <1,\\
    5 \quad \text{if } \tilde R_{ij}^{(t)} >5,\\
    \tilde R_{ij}^{(t)} \quad \text{otherwise}
    \end{cases},
\end{equation}
where $\tilde R_{ij}^{(t)} = \mathbf U_i^{(t)} \mathbf V_j^{(t)}$, $T$ is the test-set and $|T|$ denotes its size.

The rMSE of the ergodic average of the preditions (e-rMSE) at iteration $t$ is computed as 
\begin{equation}
    \label{eq:mcmc_rMSE_bpmf}
    \text{e-rMSE}(t)= \sqrt{\frac{1}{|T|} \sum_{i,j : (i, j) \in T} \left(R_{ij} - \bar{\hat{R}}_{ij}^{(t)}, 
    \right)}
\end{equation}
where $\bar{\hat{R}}_{ij}^{(t)} = \frac{1}{t} \sum_{k=1}^t \hat R_{ij}^{(k)}$ is the ergodic average of the predictions for $R_{ij}$
Similar results were obtained without clipping the predictions.

In Figure \ref{ica_plot_1} of Section \ref{subsec:ica}, the sample mean log-likelihood ($\text{s-}\mathcal{L}$) at iteration $t$ is computed as 
\begin{equation}
    \label{eq:sample_log_lik_ica}
    \text{s-}\mathcal{L}(t) = \frac{1}{|T|} \sum_{i \in T}\log \left(p(x_i \mid W^{(t)})\right)
\end{equation}
and the mean log-likelihood of the ergodic average of $W$ 
 ($\text{e-}\mathcal{L}$) at iteration $t$ is given by
\begin{equation}
    \label{eq:mcmc_log_lik_ica}
   \text{e-}\mathcal{L}(t)= \frac{1}{|T|} \sum_{i \in T}\log \left(p(x_i \mid \overline {W}^{(t)})\right)
\end{equation}
where $\overline{W}^{(t)} = \frac{1}{t}\sum_{k=1}^t W^{(k)}$, $T$ is the test-set and $|T|$ denotes its size.
We chose the following step-sizes optimizing predictive performance of the ergodic average of the samples of $W$ over a grid. obtaining $\sigma=0.0110$ for v-SGBD, $\sigma=0.0084$ for c-SGBD and e-SGBD, $\sigma=0.0070$ for v-SGLD and c-SGLD and $\sigma=0.0063$ for e-SGLD.

\end{document}